\documentclass[letterpaper]{article}

\usepackage{natbib,alifeconf}  

\usepackage{amsthm}
\usepackage{amsfonts}
\usepackage{bm}

\usepackage{algorithm}
\usepackage{algorithmic}
\usepackage{hyperref} 

\usepackage{float}
\usepackage{caption}
\usepackage{newfloat}
\usepackage{listings}
\DeclareCaptionStyle{ruled}{labelfont=normalfont,labelsep=colon,strut=off} 
\floatstyle{ruled}
\newfloat{listing}{tb}{lst}{}
\floatname{listing}{Listing}

\newtheorem{definition}{Definition}
\newtheorem{theorem}{Theorem}

\title{GO-EA in the light of SGD}

\title{Evolutionary Algorithms in the Light of SGD:\\
Limit Equivalence, Minima Flatness, and Transfer Learning}
\author{Andrei Kucharavy$^{1,*}$, Rachid Guerraoui$^{1}$ \and Ljiljana Dolamic$^2$ \\
\mbox{}\\
$^1$IC department, EPFL, Lausanne, Switzerland \\
$^2$CYD Campus, armasuisse S+T, Thun, Switzerland \\
$^*$andrei.kucharavy@hevs.ch; Now at HES-SO Valais/Wallis, Sierre, Switzerland}

\begin{document}
\maketitle

\begin{abstract}
Whenever applicable, the Stochastic Gradient Descent (SGD) has shown itself to be unreasonably effective.
Instead of underperforming and getting trapped in local minima due to the batch noise, SGD leverages it to learn to generalize better and find minima that are good enough for the entire dataset. 
This led to numerous theoretical and experimental investigations, especially in the context of Artificial Neural Networks (ANNs), leading to better machine learning algorithms.
However, SGD is not applicable in a non-differentiable setting, leaving all that prior research off the table.

In this paper, we show that a class of evolutionary algorithms (EAs) inspired by the Gillespie-Orr Mutational Landscapes model for natural evolution is formally equivalent to SGD in certain settings and, in practice, is well adapted to large ANNs. We refer to such EAs as Gillespie-Orr EA class (GO-EAs) and empirically show how an insight transfer from SGD can work for them. We then show that for ANNs trained to near-optimality or in the transfer learning setting, the equivalence also allows transferring the insights from the Mutational Landscapes model to SGD.

We then leverage this equivalence to experimentally show how SGD and GO-EAs can provide mutual insight through examples of minima flatness, transfer learning, and mixing of individuals in EAs applied to large models.
\end{abstract}

\section{Introduction}

Over the last decade and a half, deep learning has achieved impressive progress\citep{lecun2015}. From image recognition \citep{ImageNet2012} to image and text synthesis \citep{Karras2018, GPT3}, a progressive increase in the size of ANN models, combined with an increase in the size and variety of datasets and new approaches, such as GANs or Self-Attention \citep{Goodfellow2014GANs, Transformer2017}, unlocked new capabilities - both expected and unexpected \citep{Antropic2022BiggerModelsUnlockPerformanceAndRacism}. However, one of the most fundamental aspects of such models - their training process - is still far from fully understood.

The ability of ANN models to solve large classes of problems is not exactly surprising. ANNs have been proven to be universal approximators that can fit any function \citep{Hornik1989}. What is surprising is that good approximations can reliably be found by trivial gradient descent, even if it is calculated on small batches of dataset samples - a procedure known as Stochastic Gradient Descent (SGD) \cite{LecunBengioBottou1998SGDbetter}. SGD can teach ANNs robust, well-generalizing features \citep{BottouLeCun2003, BottouBousquet2008}, even when models are overparametrized enough to memorize images \citep{TwoBengiosSharpGeneralization2017, SBengioMemorization2021}. While numerous other optimizers have been proposed and successfully applied - such as Adam \citep{Adam2014}, few remain as well-explored as SGD. 

Thanks to its simplicity and unreasonable effectiveness, SGD has been explored better than any other machine learning algorithm, both theoretically and experimentally. Among other things, we now understand how SGD optimization responds to learning rate adjustment \citep{jastrzkebski2018relation, hoffer2017train}, minibatch noise \citep{ Ueda2021strengthOfMinibatchNoise, Wu2020GeneralNoiseSGD}, or additional noise injection in case of large batches \citep{xie2020diffusion, zhu2018anisotropicnoise}. Several experimental and theoretical works made clear the importance of model over-parametrization for both the loss landscape smoothing \citep{LossLandscapesVis2018} and increasing the connectedness of minima \citep{NguyenConnectivity2021, NeuralTangentKernels2018}, allowing the SGD to avoid getting trapped in local minima and to train ANNs to recognize robust and generalizing features instead.

However, SGD applies only in a fully differentiable setting. It requires a transformation of the learning problem into a continuous form and can only train fully differentiable models with layers that can support gradient back-propagation. However, a fully differentiable setting is highly limiting. For instance, a stronger, discrete version of the wildly successful soft attention architecture \citep{AttentionBengio2015, Transformer2017} - hard attention \citep{HardAttention2014Google, HardAttention2015Bengio} - cannot be trained with SGD or other optimizers requiring a differentiable setting. 

These limitations led to increasing attention to approaches that can work in a non-differentiable setting. A prime example is reinforcement learning (RL) \citep{Sutton1991ReinforcementLearning}. While reinforcement learning achieved impressive results in some settings, notably strategy games \citep{AlphaGo2017, AlphaStar2019}, RL is more poorly understood and requires extensive hyper-parameter space sweeps for each new application \citep{Stutskver2017ESasRLalternative}.

These shortcomings led to a resurgence of interest in evolutionary optimization algorithms (EA). Significantly more stable than RLs \citep{Stutskver2017ESasRLalternative}, EAs were shown to scale well and, when supplied with sufficient computational power, to outperform reinforcement learning on a range of complex tasks \citep{Stutskver2017ESasRLalternative, 2017UberGeneticAlgos}. Unfortunately, the understanding of why EAs perform so well and how their performance could be improved has been limited.

Here, we show that a relatively simple class of EAs rooted in a formalization from population genetics - Gillespie-Orr Mutational Landscapes model - approximates well SGD in theory and practice. We call that class Gillespie-Orr Evolutionary Algorithms (GO-EA) and experimentally show how, thanks to them, insight can be transferred between SGD and the Mutational Landscapes model of evolution. We demonstrate it with a well-established MNIST digit recognition task \citep{LecunBengioBottou1998SGDbetter}. 

Interestingly, such an equivalence allows us a novel insight into the long-standing \textit{Flat Minima} hypothesis in Machine Learning. The Flat Minima hypothesis postulates that suggesting that SGD trains better generalizing ANNs by finding flatter minima, given that those are the only ones to be robust to the inherent SGD minibatch noise \citep{FlatMinima1994}. While this theory has found some empirical support \citep{FlatMinimaGoodfellow2014, Keskar2016, Chaudhari2017, LossLandscapesVis2018}, it also has counter-examples \citep{TwoBengiosSharpGeneralization2017, Ueda2022LocalMaximaSGD}. Here we leverage the equivalence between GO-EA and SGD to transfer results from the Mutational Landscapes model \citep{kucharavy2018}, suggesting that Minima Flatness has more to do with the redundancy of feature recognition rather than the model generalizability. We show that this has implications for transfer learning, an increasingly prevalent paradigm, where existing models are adjusted for a new application rather than re-trained from scratch \citep{2010TransferLeareningReview, Bengio2014TransferLearning}.

Specifically, our contributions are:
\begin{itemize}
  \item \textbf{Establishing a limit equivalence} between SGD and GO-EA class in the low learning rate limit
  \item \textbf{Presenting a population size effect} as equivalent to uniform anisotropic noise in SGD and perturbed gradient descent (PGD), suggesting a new role for neutral drift
  \item \textbf{Empirically validating hypotheses} that arise from such equivalence, namely with regards to\textit{ minima flatness}, \textit{transfer learning}, evolutionary algorithms \textit{hyperparameters}, and sampled parameter \textit{update vectors mixing}
\end{itemize}

\section{Evolution in Algorithms and Genetics}

\subsection{Evolutionary Algorithms}

Evolutionary algorithms for optimization and AI have been introduced by \citet{Fogel1966} and were directly inspired by evolution understood in a strictly Darwinian sense - as a pure mutation-selection loop looking for an optimum on a fitness landscape. As additional concepts were picked up from biological evolution, EA became progressively more complex, culminating in the Genetic Algorithm\footnote{Evolutionary algorithms often have multiple conflicting names. To avoid confusion, here we adopt the taxonomy presented in \cite{2015EvolutionStrategiesReview}, and following the lead of \cite{NeuroevolutionReview2020}, use Evolutionary Algorithm as a general term and reserve Genetic Algorithm name strictly to the algorithm presented in \cite{Goldberg1989, GeneticAlgorithm1992Holland}, aka including chromosomes and recombination.} \citep{Goldberg1989, GeneticAlgorithm1992Holland} - perhaps the most widely known machine learning algorithm before the ubiquitous success SGD.

However, the Genetic Algorithm was far from the last evolutionary algorithm introduced. In the late 1990s and early 2000s, several new evolutionary algorithms were proposed. They were inspired less by biological evolution and more by heuristics to accelerate the search. Notable examples are Enforced Sub-Populations (ESP) \citep{1997ESP}, Covariance Matrix Adaptation Evolution Strategy (CMA-ES) \citep{2001CMAES}, CooperativeSynapse Neuroevolution (CoSynNE) \citep{2008CoSyNE} or Natural Evolution Strategies (NES) \citep{2008NES}. However, their use remained limited, and they were rapidly eclipsed by the success of deep learning and reinforcement learning, especially when applied to ANNs, with neuroevolution remaining a relative niche domain \citep{Neuroevolution2008}.

\subsection{Population Genetics and Mutational Landscape}

Where evolutionary algorithms drew inspiration from biological evolution, population genetics instead sought to formalize and precisely quantify it. Its first challenge was Darwin's theory of natural selection itself. The small, gradual changes that nature was supposed to select from were incompatible with observed patterns of inheritance in most organisms, as discovered by Mendel. It was not until Sir Ronald Fisher introduced his Geometric Model of Evolution, almost 70 years later, that the paradox was resolved from a theoretical standpoint \citep{fisher1930geometric}. By representing fitness relationship to traits as a scalar field, where gene variants encoded specific points, this model represented evolution as a random walk trying to ascend a fitness peak, with steps following Mendelian patterns and trait changes looking Darwinian \citep{tenaillon2014FGMUses, orr2005review}. Shortly after, \citet{Wright1932} introduced a generalization of that model with numerous fitness peaks and valleys and coined the term of \textit{fitness landscapes} to represent it, which in the context of machine learning saw their fitness inverted into the loss and became known as \textit{loss landscapes}.

The Fitness Landscapes model did not stop there, however. \citet{kimura1968genetic} realized that survival was not only determined by fitness but also by chance. A fire in the forest could easily exterminate a population of deer that acquired a mutation allowing them to better digest grass, eliminating the improved trait by pure chance. This observation led to the introduction of \textit{neutral drift} as a counter-balance to natural selection and the formulation of the nearly neutral theory of molecular evolution \citep{ohta1992nearly}. By bringing in evidence from paleontology, \citet{gould1979spandrels} made it clear that evolution is not a steady process but rather occurs in rapid bursts shortly after the environment changes. If the environment does not change, neither do organisms inhabiting it, leading to living fossils, such as horseshoe crabs in the Delaware River delta, unchanged for the last 480 million years. This became known as the \textit{adaptive bursts chain} theory of evolution \citep{lande1986dynamics}.

However, despite its refinements, the geometric model of evolution still had one major issue - the biological reality. The discovery of DNA in the 1950s meant that genetic code was a long string, with mutations only affecting a letter or a word in it at a time. While theories of evolution representing it as such were developed - notably the string rewrite graph \textit{NK theory} \citep{kauffman1969metabolic, kauffman1987towards}, they looked nothing like scalar vector fields of geometric models and lacked the quantitative explanatory capabilities of the latter. This conundrum was resolved by \citet{gillespie1983simple} and \citet{orr2002population}. The former noted that within the adaptive burst theory evolution, the adaptation would occur only after the environment would change and would start with an organism that already had a genetic code mapping to a fitness maximum within the accessible genetic code space prior to the environment change. Hence it was starting the search for mutations improving its fitness from a genetic code with an already extreme fitness compared to the ensemble of all possible genetic codes. In turn, it meant that the fitness change could be described by an extreme limit distribution, leveraging the Fisher-Tippet-Gnedenko theorem \citep{gillespie1984molecular, fisher1928limiting, gnedenko1943distribution}. At this point, \citet{orr2002population, orr2006distribution}
showed that the adaptive walk in Fisher's Geometric Models belonged to the same limit distribution family; hence, both were formally equivalent. Finally, \citet{JoyceOrr2008} showed that the underlying classes of fitness distributions across code strings mattered little - the Geometric Model still represented most heavy-tailed or truncated fitness distributions well enough.

The resulting model became known as \textit{Mutational Landscapes models} \citep{orr2005review}, or Modern Fisher Geometric Model \citep{tenaillon2014FGMUses}. Remarkably, despite being explicitly developed for biological organisms, it is well-suited for any coding space search, with a strict equivalence to the biological setting when the adaptation occurs from an already well-performing code, such as in transfer learning or the final stages of model training.

\subsection{Prior work}

Unfortunately, to our knowledge, the equivalence between evolutionary algorithms and gradient descent algorithms has remained a relatively unexplored topic.

Closest to our approach, \citet{Stutskver2017ESasRLalternative} established an informal equivalence between Q-Learning and Policy Gradient \citep{Watkins1992QLearning, Wiliam1992PolicyGradient} and a type of evolutionary search algorithm (ES) - Scalable ES, in the context of reinforcement learning. The authors speculate that Q-Learning and Policy Gradient explore possible actions by perturbing the actions of a learning agent, ES perturbs the parameters of the ANN controlling the actor's action choice directly. However, the authors stop there and proceed to experimental investigations as to whether ES could solve the RL tasks they were interested in.

Similarly, there is a wealth of papers that establish in parallel an equivalence between a machine learning process and a physical process and between a biological evolutionary process and the same physical process. For instance, both \citet{Koonin2018Glassy} and \citet{LeCun2018Glassy} draw analogies with glassy systems for biological evolution and ANN training with SGD, respectively; while failing to acknowledge that the field of glassy systems has itself been inspired by NK models developed to explain natural evolution by \citet{kauffman1987towards}. 

However, to our knowledge, none formalized the direct equivalence between SGD and a class of EAs or connected it to population genetics.

\section{Central Theoretical Results}

\subsection{Gillespie-Orr Evolutionary Algorithms Class}

While the Mutational Landscape model of evolution is a general framework, here we will focus on a simplified version that is best suited for optimization tasks and theoretical analysis. Specifically, by noting $N$ is the population size, $s$ is a typical selection coefficient, and $\mu$ is the per-site mutation rate, we will be making the following assumptions:
\begin{enumerate}
  \item Haploid populations (single code evaluated for fitness);
  \item Under high selection ($Ns \gg 1$);
  \item In the low mutation limit ($N \mu < 1$);
\end{enumerate}

The main purpose of those assumptions is to ensure that a new code modification (mutation) is evaluated by itself and has the time to become universal in the evolving population on its own merits. To enforce it, we define the Gillespie-Orr Evolutionary Algorithms class (GO-EAs) as follows:

\begin{definition}[Gillespie-Orr Evolutionary Algorithm class]
Any parameter space search algorithm that evaluates the change in loss function $\mathcal{L}_{ \theta}$ upon update of model parameters ($\bm \theta$) with a random vector of perturbation ($\mathcal{L}_{ \theta +  \theta_{rand}}$), without aggregation with other loss function evaluations, and performing a greedy search based on such evaluation.
\end{definition}

For simplicity, in the case when model parameters $\theta$ are all real numbers, for convenience, we decompose the angular component of update $\theta_{rand}$ from its scalar component $\epsilon$ and, by abuse of notation, write it as $\epsilon \theta_{rand}$. Given that $\epsilon$ indicates how much the model parameters (aka code)  can change, we refer to it as the \textit{rewrite capacity}. Conversely, given that $\theta_{rand}$ indicates the direction of a potential optimization step, we call it \textit{update vector}. If $\mathcal{L}_{ \theta +  \epsilon \theta_{rand}}<\mathcal{L}_{ \theta}$, we call $\theta_{rand}$ a \textit{valid update vector}.

Similarly, the assumptions of the absence of aggregation between different loss function evaluations and haploidy are here for simplicity. In biological evolution, for polyploid sexually reproducing organisms, this assumption is relaxed by only considering mutations that spread throughout the entire population (sweep) and by counting the generations needed for that sweep as a single generation. Given that evolutionary algorithms are not constrained by molecular biology, we abstract this away through our definition. 

This means that Scalabe ES \citep{Stutskver2017ESasRLalternative}, CoSyNE \citep{2008CoSyNE}, NES \citep{2008NES}, or the Genetic Algorithm \citep{Goldberg1989} are not part of GO-EA class due to the aggregation of parameters coming from different fitness evaluations. Conversely, a simple greedy search algorithm described in Appendix Alg.\ref{basic_go_ea} is part of the GO-EA class, just as the algorithm proposed in \citet{2017UberGeneticAlgos}.

\subsection{In the Limit, GO-EA Converges to SGD in Mean}

Using the standard notation, let $f_{\bm \theta}(\cdot)$ be an ANN parameterized by $\bm \theta$, that maps inputs $\bf X=\{\bf x_i \}^M_{i=1} \in \mathbb{Z}_2^{n_x \times d_x \times M}$ to outputs $\bf Y=\{\bf y_i \}^M_{i=1} \in \mathbb{Z}_2^{n_y \times d_y \times M}$, where $\mathbb{Z}_2=\{0, 1\}$, $d_x$ and $d_y$ are dimensions of $\bm x$ and $\bm y$, $n_x$ and $n_y$ respectively the binary code length required to describe a single component of the vectors of $\bm x$ and $\bm y$ and $M$ the maximum number of inputs the network can encounter, with potentially $M=\inf$.

Let $\mathcal{L}_{\bm \theta}$ be the fitness function associated to  $f_{\bm \theta}$ on the $\bf X$ and $\bf Y$. A priori,  $\mathcal{L}$ is inaccessible because it requires evaluating all the possible input-output pairs. However, it can be estimated with a finite sample of inputs and outputs $\bf X_{samp}$, $\bf Y_{samp}$, giving us an $\hat{\mathcal{L}}_{\bm \theta}|_{{\bf X}_{samp}, {\bf Y}_{samp}}$, that we will shorten to $\hat{\mathcal{L}}_{\bm \theta}$.

Let $\mathcal{O}$ be a greedy optimization process, such that $\mathcal{O}(\bm \theta) = \bm \theta'$,  with a parameter change capacity $d$, such that $||\bm \theta' - \bm \theta||_p < d$, where $p \in \mathbb{N}$ and $\hat{\mathcal{L}}_{\bm \theta'}|_{{\bf X}_{samp}, {\bf Y}_{samp}} \geq \hat{\mathcal{L}}_{\bm \theta''}|_{{\bf X}_{samp}, {\bf Y}_{samp}}$ for any $\bm \theta''$ such that $||\bm \theta'' - \bm \theta||_p < d$.

\begin{theorem}[Low learning rate, high population]
\label{low_lr_high_n}
GO-EA update converges in mean towards SGD gradient as sampling population increases towards infinity ($N \rightarrow \infty$), assuming a locally smooth surface ($\forall \bm \theta_1, \bm \theta_2$, such that $|\bm \theta_1 - \bm \theta'|<l$ and $|\bm \theta_2 - \bm \theta'|<l$, and $\frac{|\hat{\mathcal{L}}_{\bm \theta_1} - \hat{\mathcal{L}}_{\bm \theta_2}|}{|\theta_1 - \theta_2|}<k$ where $k$ is the Lipschitz constant of the loss surface), a non-zero gradient and a low learning rate limit ($l k=d \rightarrow 0$).
\end{theorem}

\begin{proof}
Let $\bm \theta_0$ be the starting parameters, $\bm \theta'$ be the parameter found by a step of SGD, and $\bm \theta''$ be the parameter found by a GO-EA process for a single minibatch of SGD.
Given that the loss surface is locally smooth $\nabla \hat{\mathcal{L}}_{\theta_0}$ exists and $\bm \theta' = \bm \theta_0 + l \nabla \hat{\mathcal{L}}_{\theta_0}$, where $|\nabla \hat{\mathcal{L}}_{\theta_0}|=k$ (since $\mathcal{L}$ is fitness, we are looking to ascend the gradient, as opposed to the loss where we look to descend it, and $+$ becomes a $-$). Given the small learning rate approximation, we can ignore higher order terms and $\bm \theta'$ is the argmax of $\hat{\mathcal{L}}_{\theta}$, $\forall \theta$ such that $|\theta-\theta_0|<l k = d$. Because of linearity of the exploration space, $\forall \theta_1,\theta_2$ such that $\hat{\mathcal{L}}_{\theta_2} > \hat{\mathcal{L}}_{\theta_1}$, $|\theta_1-\theta''|<|\theta_2-\theta''|$. 

Let $\bm \theta''_N$ be the best parameters found by $\mathcal{O}$ with a population $N$. Because of the above, we have $\forall N>0$,  $|\theta''_{N+1}-\theta'|\leq|\theta''_N-\theta'|$. Similarly, given that $k l = d$, $\forall \epsilon > 0$, $\exists N$ such that $|\theta''_N-\theta'|<\epsilon$. In turn, this implies that $lim_{N \rightarrow +\infty}E(|\theta''_N-\theta'|)=0$, which is a convergence in mean.
\end{proof}

\subsection{Relaxing Limit Constraints}

The low learning limit used above is known as the continuous-time approximation and has been used to obtain several theoretical results regarding SGD learning. However, large learning rates have been shown to be critical for SGD generalization \citep{Bengio2018WalkWithSGD, Ueda2021strengthOfMinibatchNoise} and the infinite population size is neither realistic nor necessary in practice.

 The minibatch noise in SGD has been shown to play an essential role in its ability to teach well-generalizing models, even if the noise distribution matters little \citep{Wu2020GeneralNoiseSGD}. One of the approaches that have been used to emulate the SGD in the case where the model and hardware allow for batches that are too large is artificial noise injection \citep{Zhou2019,  Bach2022Anticorrelatednoise}, to the point where it is possible to recover generalization properties of SGD with large batches \citep{geiping2021stochastic}.

Assuming that for large batch size, the SGD parameters update vector is close to the one of GD, and assuming that the random sampling process is anisotropic, we can easily calculate the probability of randomly sampling a vector $\theta_{rand}$ within an angle $\alpha$ of the GD update vector. If the sampling is uniform, the chance to sample such a vector would be equal to the ratio of the area of the cap of a hypersphere in dimension $n=dim(\theta)$ delimited by the colatitude angle $\alpha$ relative to the whole hypersphere surface area. Fortunately for us, this is a well-known function, mapping to the normalized incomplete beta functions $I_x(a,b)=\frac{B(x;a,b)}{B(a,b)}=\frac{\int_{0}^{x} t^{a-1} (1-t)^{b-1} \,dt }{ \int_{0}^{1} t^{a-1} (1-t)^{b-1} \,dt }$  \citep{hypercapref}. The final closed form that can be used for the estimation is
$f=\frac{1}{2} I_{sin^2\alpha}(\frac{n-1}{2}, \frac{1}{2})$.

\subsection{Improvement Probability in Transfer Learning}

In the context of fine-tuning, we expect to start with a model $f_{\bm \theta_0}(\cdot)$ parameterized so that it already performs well on all the sample tests drawn from the distribution it was used to train with - aka $\forall ({\bf X}_{samp}, {\bf Y}_{samp}) \subset \bf X \times \bf Y$, 
$\mathbb{P}(\hat{\mathcal{L}}_{\bm \theta_0}|_{{\bf X}_{samp}, {\bf Y}_{samp}} \sim \max_{\bm \theta}{\hat{\mathcal{L}}_{\bm \theta}|_{{\bf X}_{samp}, {\bf Y}_{samp}}}) \sim 1$. Formally, transfer learning consists in finding a new transfer parametrization $\bm \theta_T$, so that $\forall ({\bf X}_{samp}, {\bf Y}_{samp}) \subset \bf X 	\cup \bf X' \times \bf Y 	\cup Y'$,
$\mathbb{P}(\hat{\mathcal{L}}_{\bm \theta_T}|_{{\bf X}_{samp}, {\bf Y}_{samp}} \sim \max_{\bm \theta}{\hat{\mathcal{L}}_{\bm \theta}|_{{\bf X}_{samp}, {\bf Y}_{samp}}}) \sim 1$, where the $\bf X'$ and $\bf Y'$ are new domains application of the model.

Assuming $|\bm X|>>|\bm X'|$ and $|\bm Y|>>|\bm Y'|$, since otherwise, transfer learning would be equivalent to model re-training, the model is already performing well on the transfer model and the vast majority of the parameters within rewrite capacity $d$ of $\bm \theta_0$ would be deleterious or neutral, meaning that the parametrizations offering improvement would be distributed according to the generalized Pareto distribution \citep{Pickands1975, JoyceOrr2008}, which in the case of Gumbel domain of attraction would result in an exponential distribution  of fitnesses $\bm s = (s_{1}, ..., s_{i-1})$ where $s_j=\hat{\mathcal{L}}_{\bm \theta_j}|_{{\bf X}_{samp}, {\bf Y}_{samp}}$, the $j^{th}$ best parametrization of better parametrizations and a probability to reach the better parametrization $\bm \theta_j$ of rank $j$ in the neighborhood from a parametrization $\bm \theta_i$ of the rank $i$ of 
$ \mathbb{P}_{i,j}(\bm s) = \frac{s_j}{\sum^{i-1}_{k=1}s_k}$. In other terms, with finite populations, GO-EA sampling the parametrization neighborhood of the current optimum $\bm \theta_i$ will find advantageous model code rewrites with the probability that is inverse to the exponential probability of the difference between the loss associated to $\bm \theta_i$ and smallest possible loss within the edit distance budget.

\section{Hypotheses based on central results}

\subsection{Overparameterized Setting}

Previously, we demonstrated a formula to calculate the chance of random search finding a good approximation of the GD vector. However, if we visualize that function in different dimensions (Appendix Fig.\ref{fig:beta_inc_sweep}), we see that coming even within 30 degrees of the GD updates with GO-EA is unrealistic in any dimension above ~100. Modern ANNs with thousands of parameters on the lower end would require sampling populations too large for any practical use.

However, most current ANN models are highly over-parametrized to stabilize their learning. The over-parametrization has been shown to smooth the loss landscape \citep{LossLandscapesVis2018} and connect minima \citep{NguyenConnectivity2019, NeuralTangentKernels2018}, allowing most strategies performing gradient descent to arrive at an acceptable minimum. Empirical investigations into the update vectors of SGD in this context \citep{Bengio2018WalkWithSGD} suggest that SGD updates are rarely aligned, can be orthogonal, and often point in opposite directions due to the fact that almost all of the points in the loss landscape model passes through are saddle points.

If this is indeed the case, then numerous vectors are "valid" in the sense that they correspond to updates that could result from SGD minibatches from the training dataset that would still allow convergence. At this point, random sampling only needs to land close to a "valid" update vector, effectively decreasing the required sampling population size. Informally, we expect random sample vectors that lead to a lower large batch loss to be close to SGD update components that do not cancel out and are not that rare. In fact, in the most extreme case, when SGD minibatch updates are orthogonal, the curse of the dimension is lifted, and the sampling size required to find a valid update is divided by the number of minibatches. Intuitively, large and diverse datasets applied to overparametrized models will lead to a reasonably fast convergence, even with small search populations. 

Because of that, we hypothesize that the training hyperparameters from SGD are directly translatable to the hyperparameters of the GO-EA class algorithm and allow us to train a model with GO-EA while using a relatively small number of samples per step.

We validate this hypothesis by first verifying that for our model ANN training task, SGD update vectors are indeed highly dispersed and showing that a basic GO-EA algorithm (Appendix Alg.\ref{basic_go_ea}) with a small population can efficiently train a model with hyperparameters copied over directly from SGD.

\subsection{Minima Flatness as Error Correction Redundancy}

The convergence of SGD training to a flat minimum for a model is believed to be one of the conditions for the model training stability \citep{LossLandscapesVis2018}. The minima flatness was assumed to be connected to their generalization abilities through the minimal coding length of the model \citep{FlatMinima1997, FlatMinimaGoodfellow2014}. However, recent evidence argues to the contrary \citep{TwoBengiosSharpGeneralization2017, IdentityCrisisBengio2020, MulayoffM20}. Within the theory of evolution, the flatness of the fitness peak is commonly associated with the tolerance to the neutral drift - aka error correction capabilities. By using this analogy, we suggest that, just like in the context of the evolution, the flatness of the loss function minimum in ANNs optimized through SGD is determined by the redundancy of the features used by the trained ANN to recognize patterns in the target data. 

This intuition seems to be consistent with empirical observations about the loss function minima flatness. Architectures that provide the model with the means to encode redundant features, such as wide hidden layers or skip-forwards connections in deep ConvNets, contribute to making the loss landscape minima flatter \citep{LossLandscapesVis2018}. Similarly, drop-out regularization \citep{DropOut2014}, forcing the ANNs to learn redundant, error-correcting codings, seems to flatten minima as well, along with the smaller batches, which can contain a larger proportion of samples that defy the heuristics that the ANN has learned before encountering them \citep{FlatMinimaGoodfellow2014}. This hypothesis, in particular, goes against the suggestion that for EAs, the search population scales only with latent dimension, \citep{Stutskver2017ESasRLalternative}. To allow robustness through redundant error correction, the actual parameter dimension matters. We validate this intuition experimentally.

\subsection{Flat Minima and Transfer Learning}

Building on the hypothesis presented above, if the minima flatness is indeed related to the classification robustness and error correction, we expect models that learned a variety of error-correcting representations of training data to not be able to transfer those representations without training on the new data. 

Intuitively, they rely on parallel redundant subpaths traversing ANN layers to detect redundant relevant features present in the training dataset. With only some of those features present in the dataset on which the transfer task is performed, their error correction property could interfere with the corresponding output without an expected degree of redundant detection. Similarly, we do not expect flatter minima to accelerate transfer learning.

While some minima sharpening is observed experimentally during transfer learning, the results in the model we used are inconclusive, and more investigation is warranted. As for the speed of transfer learning, our experiments suggest flatter minima beyond base performance improvement do not accelerate that due to redundant coding. Experimental results supporting this hypothesis are provided further and in the Appendix.

\subsection{Update Vector Mixing is Unnecessary}

One of the prominent features of most recent Evolutionary Algorithms, ranging from the Genetic Algorithm and NES to the Scalable ES, is the mixing of different vectors that were used to sample the loss landscape to generate a new step. This approach is based on the intuition that vectors found by sampling are approximations of a "true" empirical gradient descent vector, and by combining them, the empirical gradient can be better approximated. 

Our formalization of GO-EA suggests that this is likely not the case. Valid update vectors are unlikely to be aligned and are potentially orthogonal, meaning that averaging them out is counter-productive, in the same way as increasing batch size is in SGD. There is a priori no reason why any interpolation between valid update vectors in EA would be a valid update vector itself, let alone result in a lower loss than either of the valid update vectors. In fact, empirical studies have shown that minima interpolations tend to perform poorly \citep{LossLandscapesVis2018}. Similarly, the justification of the Genetic Algorithm's chromosome cross-over - allowing the beneficial mutations to combine and eliminate deleterious ones - does not apply in the setting where the change of generations is not mandatory, and hence Muller's Ratchet cannot occur \citep{lynch1993mutational}. The GO-EA class defined here goes around the problem altogether by evaluating only a single modification to the model parameters at a time. We observe \& validate the hypothesis experimentally.

\section{Experimental results}

Appendix, supplementary figures, and the code used for experiments presented here are available from the GitHub repository of the project - \href{https://github.com/chiffa/ALIFE2023_GOEA-SGD}{https://github.com/chiffa/ALIFE2023\_GOEA-SGD}.

\subsection{Model Used}

In order to perform numerical experiments, we used a convolutional neural network (ConvNet) learning to recognize digits in the MNIST dataset \citep{LecunBengioBottou1998SGDbetter}. It is a well-established model and a textbook use case for SGD, chosen to minimize the chance of unexpected edge cases interfering with our experimental results.
The detailed ANN architecture and hyperparameters are available in the Appendix.
To measure the flatness of minima and smoothness of loss landscapes consistently with prior work, we performed a spectral normalization of each layer so that a perturbation along a random axis would correspond to the local robustness of the model. This approach is strictly equivalent to the filter-wise normalized directions proposed in \citet{LossLandscapesVis2018}.

For clarity, we are using the following abbreviation for hyperparameters and architecture models: batch size as B, drop-out as DO, drop-out on inputs as DI, with two architectural parameters: latent maps in the first layer (LM) - all subsequent follow a predefined ration, and linear width (LW), determining the number of neurons in the hidden linear layer. For flatness experiments, we use three stereotypical settings: "Robust Wide" (DO:0.25; DI:0.1; LM:16; LW:48; B:4) with high loss landscape smoothness/flat minima; "Brittle Narrow" (-DO; -DI; LM:4; LW:12; B:128) with low loss landscape smoothness; and "Brittle Wide" (-DO; -DI; LM:16; LW:48; B:128), added to make sure the instability of Brittle Narrow was not due to the lack of latent features/available feature encoding space. We confirmed the model validity by replicating results from \citet{LossLandscapesVis2018}, as described in the Appendix, notably the Appendix Fig.\ref{flat_pre_fine_tune}. The "Robust Wide" model indeed maps to flatter minima, whereas both "Brittle" models map to sharper minima.

\subsection{Modeling the transfer learning}

We used the transfer learning setting to evaluate the trained model's generalization ability. We expect that a model that generalizes better would be able to perform better when encountering new data or at least leverage more general features it learned to learn new data faster. We do not expect such abilities from models with simply a more redundant feature encoding, allowing us to differentiate the two in the context of minima flatness.

For that, we performed partial training dataset occlusion, where classes in the training dataset were occluded with respect to the model and loss calculation. In the transfer learning phase, the occluded categories were revealed and included in the loss computation.


\subsection{SGD Minibatch Noise is High}

To evaluate the minibatch noise for different batch sizes, we trained with occlusion the transfer model with default hyperparameters (-DO; -DI; LM:8; LW:12; B:32), froze its parameters and evaluated the angle between updates resulting from different batch sizes, ranging from 1024 to 4, shown on Fig.\ref{fig:empirical_sgd_angles}. As we hypothesized, the batch noise increases as batch size decreases, with a modal angle of ~80 degrees for minibatches of 4 samples. Similarly, large batches (1024) are collinear and approximate GD well. Given that the base model is approximately 10k parameters, 60k samples MNIST training dataset provides approximately 15k different sample batches of 4, suggesting that a small sampling population of 10-100 individuals would be sufficient.

\begin{figure}[H]
\vskip 0.2in
\begin{center}
\centerline{\includegraphics[width=0.85\columnwidth]{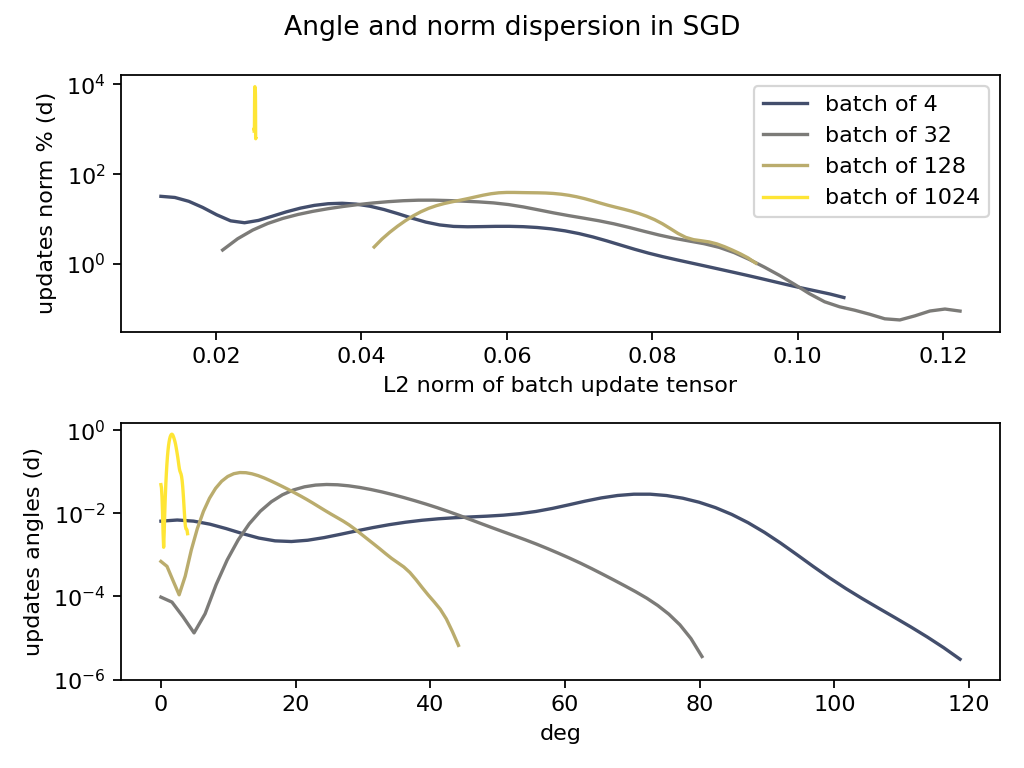}}
\caption{Empirical evaluation of angles between possible SGD update vectors for the same network state and data}
\label{fig:empirical_sgd_angles}
\end{center}
\vskip -0.2in
\end{figure}

\subsection{GO-EA Trains Efficiently with\\SGD Hyperparameters}

\begin{figure}[h]
\vskip 0.2in
\begin{center}
\centerline{\includegraphics[width=0.85\columnwidth]{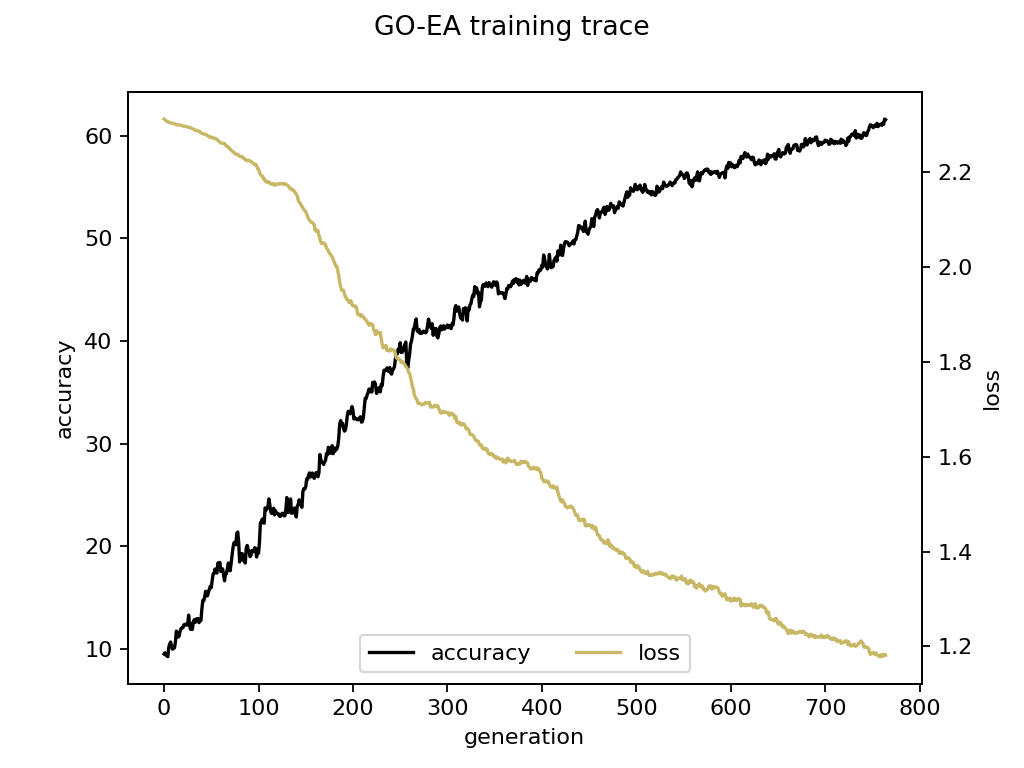}}
\caption{Loss and accuracy of the 9854 parameters ConvNet trained on MNIST dataset with a sampling population of 20 trained by a basic GO-EA class algorithm.}
\label{go_ea_training}
\end{center}
\vskip -0.2in
\end{figure}

We ported the hyperparameters and model used with the SGD optimizer directly to a simple algorithm in the GO-EA class (Appendix Alg.\ref{basic_go_ea}) and chose the sampling population size of 20, as per the previous section. The model could train rapidly, achieving an accuracy of over 60\% after 800 sampling steps (Fig.\ref{go_ea_training}). Further hyperparameter tweaking had little effect on training speed, with the sampling population size having only a moderate effect (Appendix Fig.\ref{go_ea_sweep}), suggesting that SGD hyperparameters did indeed transfer well.

\subsection{Flat Minima don't Generalize Better,\\They are More Robust}

After training the three minima flatness stereotype models according to the transfer learning model, we saw a significant difference in the flatness of their eight classes minima (Appendix Fig.\ref{flat_8_cl}). Models with flatter minima do not seem to generalize better (Appendix Fig.\ref{flat_fine_tune_speed}), given that they do not perform better on new data, nor do they undergo transfer learning faster beyond better performance on already known classes. We observe, however, significantly more redundant feature encoding, making them more resilient to noise in input images (Fig.\ref{rob_gen}).

\begin{figure}[H]
\centering
\includegraphics[width=0.9\columnwidth]{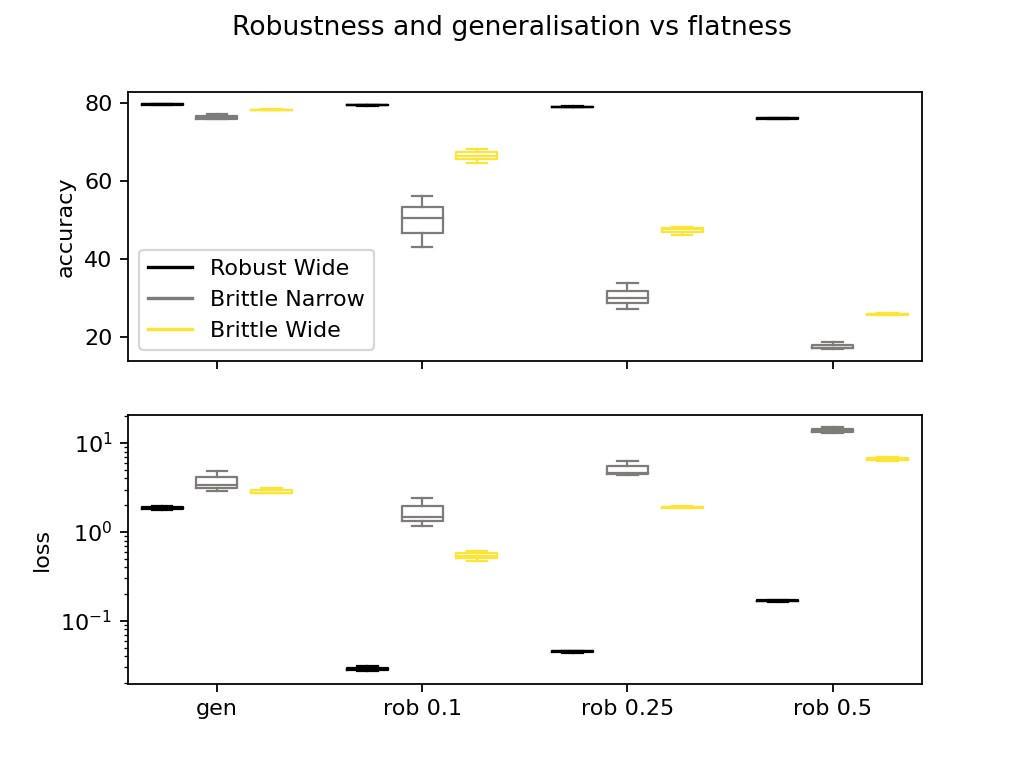} 
\caption{Flatter Minima models are more noise resistant but do not generalize better. Archetype models pre-trained on eight classes with occlusion generalization on the whole dataset (gen) as well as robustness to 10,25, and 50\% source image and feature corruption (rob 0.1, 0.25, and 0.5).}.
\label{rob_gen}
\end{figure}

\subsection{Transfer Learning Conforms to\\Mutational Landscapes Models}

\begin{figure}[h]
\centering
\includegraphics[width=0.85\columnwidth]{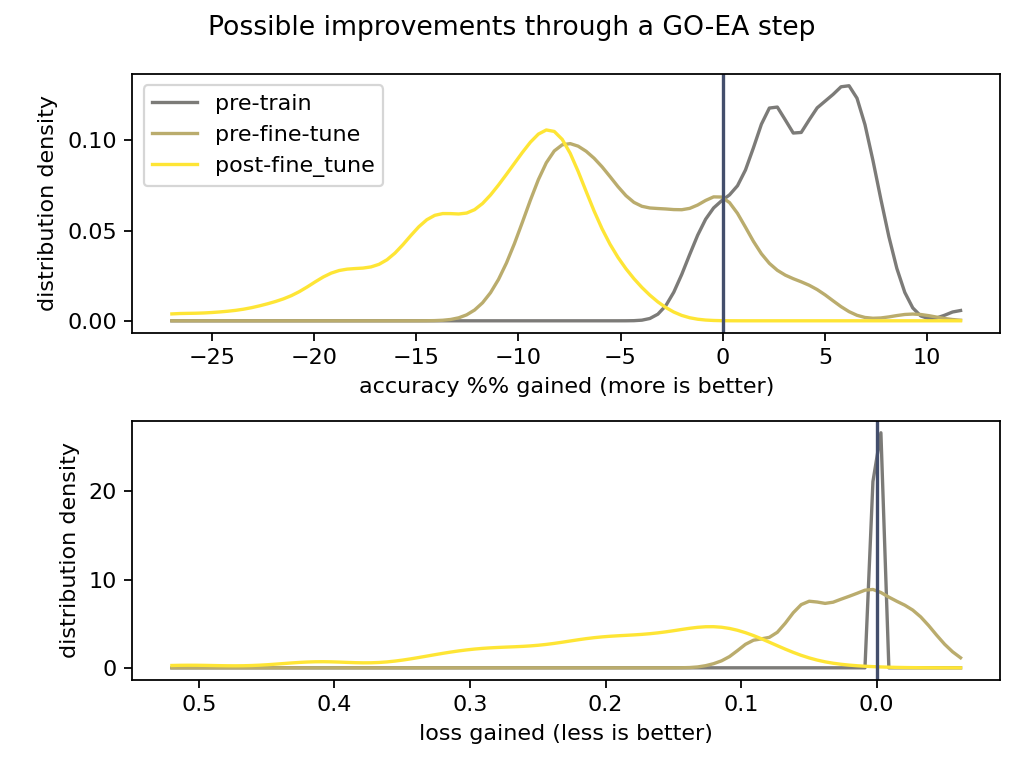} 
\caption{Probability of random update step being beneficial or deleterious as well as the magnitude of effect in accuracy (top panel, right is better) and in loss (bottom panel, left is better)}.
\label{fitness_distro_gain_fine_tune}
\end{figure}

To verify that ANNs trained with SGD conformed to the Mutation Landscape Model during the transfer learning process, we performed a random axis sweep of mutations with an edit distance of the order of magnitude of the standard deviation of the norm of weights in each layer. As shown in Fig.\ref{fitness_distro_gain_fine_tune}, it indeed does. We also observe that in conformity to \citet{orr2005review}, at random initialization, around 50\% of directions result in improvement (assumption of NK models, unrealistic in population genetics). In contrast, at the optimum, all mutations are deleterious, and the start of adaptive burst (start of transfer learning) conforms perfectly to the Mutational Landscapes Model (cf Fig.3 in \citet{orr2005review}).

\subsection{Update Vector Mixing is Likely Undesirable}

In conformity with our hypothesis, valid update vectors sampled by EA are orthogonal and indistinguishable from a pair of random sample vectors with a Kolmogorov-Smirnov test between the two. Fig.\ref{orthogonal_angles} gives distributions of both in different settings and the p-value of the two-sample KS test for angles between random update vectors and between update vectors resulting in a better loss.

\begin{figure}[h]
\centering
\includegraphics[width=0.85\columnwidth]{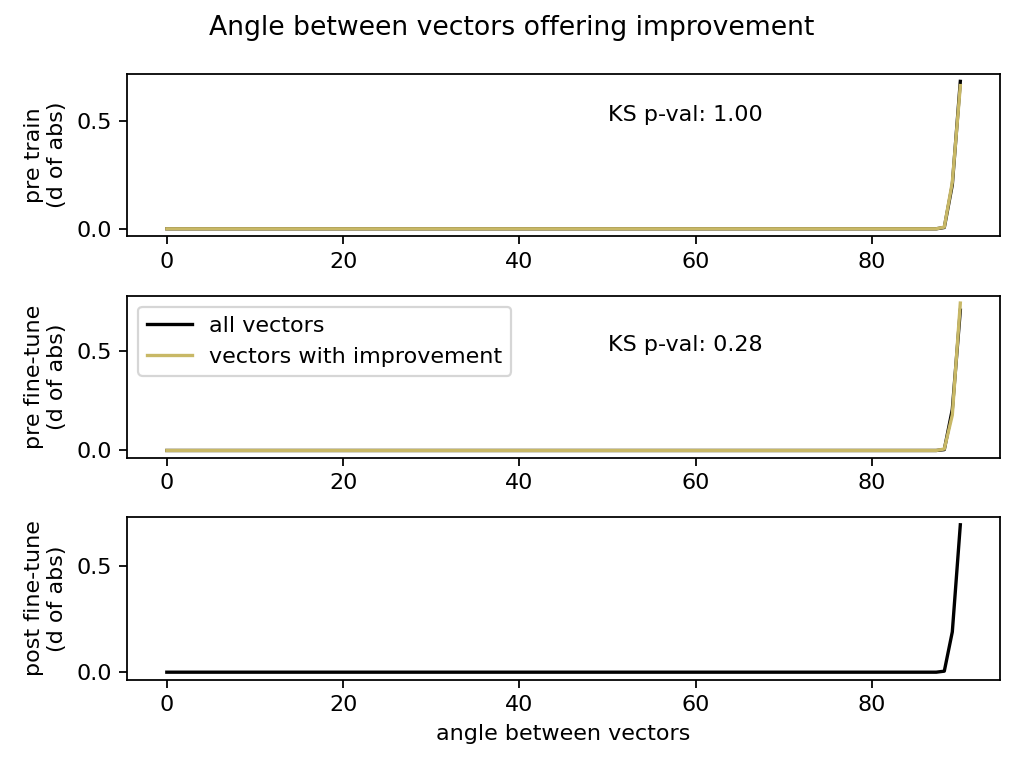}
\caption{Distribution of angles between randomly sampled update vectors leading to better fitness (yellow) vs. all randomly sampled update vectors (black), from random initialization, transfer learning start point, and at optimum (panels 1-3, respectively)}.
\label{orthogonal_angles}
\end{figure}

\section{Discussion}
In this paper, we define a new class of Evolutionary Algorithms - GO-EA, and establish a formal equivalence between them and SGD, both in the limit and in a more realistic setting. We then empirically test hypotheses arising from such equivalence for several well-established problems in machine learning, namely flat minimas, transfer learning, and applicability of EAs to modern highly over-parametrized ML models, demonstrating the advantage of the simpler GO-EA class and partially explaining the past success of GO-EA class algorithms in neuroevolution on hard tasks \citep{2017UberGeneticAlgos}.

However, this equivalence opens up a possibility to formalize and transfer other insights between machine learning and the theory of evolution. For instance, results by \citet{kucharavy2018, tenaillon2007quantifying} suggest that it is possible to evaluate the latent dimension of a problem directly. Conversely, insight into loss landscapes potentially can be translated back to the theory of evolution to explain adaptation, e.g., the feasibility of evolutionary traps to avoid drug resistance \citep{chen2015EvoTraps}.

Finally, leveraging the existing knowledge regarding SGD in a non-differentiable setting opens up new possibilities for training large models with non-differentiable layers. Specifically, hard attention is generally considered superior to soft attention, which was popularized by the Transformer and is underlying the ubiquitous Large Language Models (LLMs) \citep{HardAttention2015Bengio, Transformer2017}. Allowing hard attention in LLMs through GO-EAs has the potential to further the progress in that field.

\section{Acknowledgements}
We would like to thank the Cyber-Defence Campus, armasuisse W+T, VBS for the Distinguished CYD Post-Doctoral Fellowship to AK (ARAMIS CYD-F-2021004), as well as Fabien Salvi and France Faille (EFPL) for their technical and administrative support and the anonymous reviewers for their useful feedback.

\footnotesize
\bibliographystyle{apalike}
\bibliography{example} 

\begin{thebibliography}{}

\bibitem[Bahdanau et~al., 2015]{AttentionBengio2015}
Bahdanau, D., Cho, K., and Bengio, Y. (2015).
\newblock Neural machine translation by jointly learning to align and
  translate.
\newblock In Bengio, Y. and LeCun, Y., editors, {\em 3rd International
  Conference on Learning Representations, {ICLR} 2015, San Diego, CA, USA, May
  7-9, 2015, Conference Track Proceedings}.

\bibitem[Baity{-}Jesi et~al., 2018]{LeCun2018Glassy}
Baity{-}Jesi, M., Sagun, L., Geiger, M., Spigler, S., Arous, G.~B., Cammarota,
  C., LeCun, Y., Wyart, M., and Biroli, G. (2018).
\newblock Comparing dynamics: Deep neural networks versus glassy systems.
\newblock In Dy, J.~G. and Krause, A., editors, {\em Proceedings of the 35th
  International Conference on Machine Learning, {ICML} 2018}, volume~80 of {\em
  Proceedings of Machine Learning Research}, pages 324--333. {PMLR}.

\bibitem[Bottou and Bousquet, 2007]{BottouBousquet2008}
Bottou, L. and Bousquet, O. (2007).
\newblock The tradeoffs of large scale learning.
\newblock In Platt, J.~C., Koller, D., Singer, Y., and Roweis, S.~T., editors,
  {\em Advances in Neural Information Processing Systems 20, NeurIPS 2007},
  pages 161--168. Curran Associates, Inc.

\bibitem[Bottou and Cun, 2003]{BottouLeCun2003}
Bottou, L. and Cun, Y. (2003).
\newblock Large scale online learning.
\newblock {\em Advances in neural information processing systems}, 16.

\bibitem[Brown et~al., 2020]{GPT3}
Brown, T.~B., Mann, B., Ryder, N., Subbiah, M., Kaplan, J., Dhariwal, P.,
  Neelakantan, A., Shyam, P., Sastry, G., Askell, A., Agarwal, S.,
  Herbert-Voss, A., Krueger, G., Henighan, T., Child, R., Ramesh, A., Ziegler,
  D.~M., Wu, J., Winter, C., Hesse, C., Chen, M., Sigler, E., Litwin, M., Gray,
  S., Chess, B., Clark, J., Berner, C., McCandlish, S., Radford, A., Sutskever,
  I., and Amodei, D. (2020).
\newblock Language models are few-shot learners.
\newblock cite arxiv:2005.14165Comment: 40+32 pages.

\bibitem[Chaudhari et~al., 2017]{Chaudhari2017}
Chaudhari, P., Choromanska, A., Soatto, S., LeCun, Y., Baldassi, C., Borgs, C.,
  Chayes, J.~T., Sagun, L., and Zecchina, R. (2017).
\newblock Entropy-sgd: Biasing gradient descent into wide valleys.
\newblock In {\em 5th International Conference on Learning Representations,
  {ICLR} 2017, Toulon, France, April 24-26, 2017, Conference Track
  Proceedings}. OpenReview.net.

\bibitem[Chen et~al., 2015]{chen2015EvoTraps}
Chen, G., Mulla, W.~A., Kucharavy, A., Tsai, H.-J., Rubinstein, B., Conkright,
  J., McCroskey, S., Bradford, W.~D., Weems, L., Haug, J.~S., et~al. (2015).
\newblock Targeting the adaptability of heterogeneous aneuploids.
\newblock {\em Cell}, 160(4):771--784.

\bibitem[Dinh et~al., 2017]{TwoBengiosSharpGeneralization2017}
Dinh, L., Pascanu, R., Bengio, S., and Bengio, Y. (2017).
\newblock Sharp minima can generalize for deep nets.
\newblock In Precup, D. and Teh, Y.~W., editors, {\em Proceedings of the 34th
  International Conference on Machine Learning, {ICML} 2017}, volume~70 of {\em
  Proceedings of Machine Learning Research}, pages 1019--1028. {PMLR}.

\bibitem[Fisher, 1930]{fisher1930geometric}
Fisher, R.~A. (1930).
\newblock {\em The genetical theory of natural selection}.
\newblock Oxford University Press, Oxford.

\bibitem[Fisher and Tippett, 1928]{fisher1928limiting}
Fisher, R.~A. and Tippett, L. H.~C. (1928).
\newblock Limiting forms of the frequency distribution of the largest or
  smallest member of a sample.
\newblock In {\em Mathematical proceedings of the Cambridge philosophical
  society}, volume~24, pages 180--190. Cambridge University Press.

\bibitem[Floreano et~al., 2008]{Neuroevolution2008}
Floreano, D., Dürr, P., and Mattiussi, C. (2008).
\newblock Neuroevolution: from architectures to learning.
\newblock {\em Evol. Intell.}, 1(1):47--62.

\bibitem[Fogel et~al., 1966]{Fogel1966}
Fogel, L., Owens, A., and Walsh, M. (1966).
\newblock {\em Artificial intelligence through simulated evolution}.
\newblock Wiley, Chichester, WS, UK.

\bibitem[Galv{\'a}n and Mooney, 2021]{NeuroevolutionReview2020}
Galv{\'a}n, E. and Mooney, P. (2021).
\newblock Neuroevolution in deep neural networks: Current trends and future
  challenges.
\newblock {\em IEEE Transactions on Artificial Intelligence}, 2(6):476--493.

\bibitem[Ganguli et~al.,
  2022]{Antropic2022BiggerModelsUnlockPerformanceAndRacism}
Ganguli, D., Hernandez, D., Lovitt, L., DasSarma, N., Henighan, T., Jones, A.,
  Joseph, N., Kernion, J., Mann, B., Askell, A., Bai, Y., Chen, A., Conerly,
  T., Drain, D., Elhage, N., Showk, S.~E., Fort, S., Hatfield{-}Dodds, Z.,
  Johnston, S., Kravec, S., Nanda, N., Ndousse, K., Olsson, C., Amodei, D.,
  Amodei, D., Brown, T.~B., Kaplan, J., McCandlish, S., Olah, C., and Clark, J.
  (2022).
\newblock Predictability and surprise in large generative models.
\newblock {\em CoRR}, abs/2202.07785.

\bibitem[Geiping et~al., 2021]{geiping2021stochastic}
Geiping, J., Goldblum, M., Pope, P.~E., Moeller, M., and Goldstein, T. (2021).
\newblock Stochastic training is not necessary for generalization.
\newblock {\em arXiv preprint arXiv:2109.14119}.

\bibitem[Gillespie, 1983]{gillespie1983simple}
Gillespie, J.~H. (1983).
\newblock A simple stochastic gene substitution model.
\newblock {\em Theoretical population biology}, 23(2):202--215.

\bibitem[Gillespie, 1984]{gillespie1984molecular}
Gillespie, J.~H. (1984).
\newblock Molecular evolution over the mutational landscape.
\newblock {\em Evolution}, pages 1116--1129.

\bibitem[Gnedenko, 1943]{gnedenko1943distribution}
Gnedenko, B. (1943).
\newblock Sur la distribution limite du terme maximum d'une serie aleatoire.
\newblock {\em Annals of mathematics}, pages 423--453.

\bibitem[Goldberg, 1989]{Goldberg1989}
Goldberg, D.~E. (1989).
\newblock {\em Genetic Algorithms in Search, Optimization and Machine
  Learning}.
\newblock Addison-Wesley, Reading, MA.

\bibitem[Gomez and Miikkulainen, 1997]{1997ESP}
Gomez, F.~J. and Miikkulainen, R. (1997).
\newblock Incremental evolution of complex general behavior.
\newblock {\em Adapt. Behav.}, 5(3-4):317--342.

\bibitem[Gomez et~al., 2008]{2008CoSyNE}
Gomez, F.~J., Schmidhuber, J., and Miikkulainen, R. (2008).
\newblock Accelerated neural evolution through cooperatively coevolved
  synapses.
\newblock {\em J. Mach. Learn. Res.}, 9:937--965.

\bibitem[Goodfellow et~al., 2014]{Goodfellow2014GANs}
Goodfellow, I.~J., Pouget{-}Abadie, J., Mirza, M., Xu, B., Warde{-}Farley, D.,
  Ozair, S., Courville, A.~C., and Bengio, Y. (2014).
\newblock Generative adversarial nets.
\newblock In Ghahramani, Z., Welling, M., Cortes, C., Lawrence, N.~D., and
  Weinberger, K.~Q., editors, {\em Advances in Neural Information Processing
  Systems 27: NeurIPS 2014}, pages 2672--2680.

\bibitem[Goodfellow and Vinyals, 2015]{FlatMinimaGoodfellow2014}
Goodfellow, I.~J. and Vinyals, O. (2015).
\newblock Qualitatively characterizing neural network optimization problems.
\newblock In Bengio, Y. and LeCun, Y., editors, {\em 3rd International
  Conference on Learning Representations, {ICLR} 2015, San Diego, CA, USA, May
  7-9, 2015, Conference Track Proceedings}.

\bibitem[Gould and Lewontin, 1979]{gould1979spandrels}
Gould, S.~J. and Lewontin, R.~C. (1979).
\newblock The spandrels of san marco and the panglossian paradigm: a critique
  of the adaptationist programme.
\newblock {\em Proceedings of the royal society of London. Series B. Biological
  Sciences}, 205(1161):581--598.

\bibitem[Hansen et~al., 2015]{2015EvolutionStrategiesReview}
Hansen, N., Arnold, D.~V., and Auger, A. (2015).
\newblock Evolution strategies.
\newblock In Kacprzyk, J. and Pedrycz, W., editors, {\em Springer Handbook of
  Computational Intelligence}, Springer Handbooks, pages 871--898. Springer.

\bibitem[Hansen and Ostermeier, 2001]{2001CMAES}
Hansen, N. and Ostermeier, A. (2001).
\newblock Completely derandomized self-adaptation in evolution strategies.
\newblock {\em Evol. Comput.}, 9(2):159--195.

\bibitem[Hochreiter and Schmidhuber, 1994]{FlatMinima1994}
Hochreiter, S. and Schmidhuber, J. (1994).
\newblock Simplifying neural nets by discovering flat minima.
\newblock In Tesauro, G., Touretzky, D.~S., and Leen, T.~K., editors, {\em
  Advances in Neural Information Processing Systems 7: {[NIPS} 1994]}, pages
  529--536. {MIT} Press.

\bibitem[Hochreiter and Schmidhuber, 1997]{FlatMinima1997}
Hochreiter, S. and Schmidhuber, J. (1997).
\newblock Flat minima.
\newblock {\em Neural Comput.}, 9(1):1--42.

\bibitem[Hoffer et~al., 2017]{hoffer2017train}
Hoffer, E., Hubara, I., and Soudry, D. (2017).
\newblock Train longer, generalize better: closing the generalization gap in
  large batch training of neural networks.
\newblock {\em Advances in neural information processing systems}, 30.

\bibitem[Holland, 1992]{GeneticAlgorithm1992Holland}
Holland, J.~H. (1992).
\newblock {\em Adaptation in Natural and Artificial Systems: An Introductory
  Analysis with Applications to Biology, Control, and Artificial Intelligence}.
\newblock {MIT} Press.

\bibitem[Hornik et~al., 1989]{Hornik1989}
Hornik, K., Stinchcombe, M., and White, H. (1989).
\newblock Multilayer feedforward networks are universal approximators.
\newblock {\em Neural Networks}, 2(5):359 -- 366.

\bibitem[Jacot et~al., 2018]{NeuralTangentKernels2018}
Jacot, A., Hongler, C., and Gabriel, F. (2018).
\newblock Neural tangent kernel: Convergence and generalization in neural
  networks.
\newblock In Bengio, S., Wallach, H.~M., Larochelle, H., Grauman, K.,
  Cesa{-}Bianchi, N., and Garnett, R., editors, {\em Advances in Neural
  Information Processing Systems 31: NeurIPS 2018}, pages 8580--8589.

\bibitem[Jastrzebski et~al., 2018]{jastrzkebski2018relation}
Jastrzebski, S., Kenton, Z., Ballas, N., Fischer, A., Bengio, Y., and Storkey,
  A. (2018).
\newblock On the relation between the sharpest directions of dnn loss and the
  sgd step length.
\newblock {\em arXiv preprint arXiv:1807.05031}.

\bibitem[Joyce et~al., 2008]{JoyceOrr2008}
Joyce, P., Rokyta, D.~R., Beisel, C.~J., and Orr, H.~A. (2008).
\newblock A general extreme value theory model for the adaptation of dna
  sequences under strong selection and weak mutation.
\newblock {\em Genetics}, 180(3):1627--1643.

\bibitem[Karras et~al., 2018]{Karras2018}
Karras, T., Aila, T., Laine, S., and Lehtinen, J. (2018).
\newblock Progressive growing of gans for improved quality, stability, and
  variation.
\newblock In {\em 6th International Conference on Learning Representations,
  {ICLR} 2018, Vancouver, BC, Canada, April 30 - May 3, 2018, Conference Track
  Proceedings}. OpenReview.net.

\bibitem[Katsnelson et~al., 2018]{Koonin2018Glassy}
Katsnelson, M.~I., Wolf, Y.~I., and Koonin, E.~V. (2018).
\newblock Towards physical principles of biological evolution.
\newblock {\em Physica Scripta}, 93(4):043001.

\bibitem[Kauffman and Levin, 1987]{kauffman1987towards}
Kauffman, S. and Levin, S. (1987).
\newblock Towards a general theory of adaptive walks on rugged landscapes.
\newblock {\em Journal of theoretical Biology}, 128(1):11--45.

\bibitem[Kauffman, 1969]{kauffman1969metabolic}
Kauffman, S.~A. (1969).
\newblock Metabolic stability and epigenesis in randomly constructed genetic
  nets.
\newblock {\em Journal of theoretical biology}, 22(3):437--467.

\bibitem[Keskar et~al., 2017]{Keskar2016}
Keskar, N.~S., Mudigere, D., Nocedal, J., Smelyanskiy, M., and Tang, P. T.~P.
  (2017).
\newblock On large-batch training for deep learning: Generalization gap and
  sharp minima.
\newblock In {\em 5th International Conference on Learning Representations,
  {ICLR} 2017, Toulon, France, April 24-26, 2017, Conference Track
  Proceedings}. OpenReview.net.

\bibitem[Kimura, 1968]{kimura1968genetic}
Kimura, M. (1968).
\newblock Genetic variability maintained in a finite population due to
  mutational production of neutral and nearly neutral isoalleles.
\newblock {\em Genetics research}, 11(3):247--270.

\bibitem[Kingma and Ba, 2015]{Adam2014}
Kingma, D.~P. and Ba, J. (2015).
\newblock Adam: {A} method for stochastic optimization.
\newblock In Bengio, Y. and LeCun, Y., editors, {\em 3rd International
  Conference on Learning Representations, {ICLR} 2015, San Diego, CA, USA, May
  7-9, 2015, Conference Track Proceedings}.

\bibitem[Krizhevsky et~al., 2012]{ImageNet2012}
Krizhevsky, A., Sutskever, I., and Hinton, G.~E. (2012).
\newblock Imagenet classification with deep convolutional neural networks.
\newblock In {\em Advances in neural information processing systems}, pages
  1097--1105.

\bibitem[Kucharavy et~al., 2018]{kucharavy2018}
Kucharavy, A., Rubinstein, B., Zhu, J., and Li, R. (2018).
\newblock Robustness and evolvability of heterogeneous cell populations.
\newblock {\em Molecular biology of the cell}, 29(11):1400--1409.

\bibitem[Lande, 1986]{lande1986dynamics}
Lande, R. (1986).
\newblock The dynamics of peak shifts and the pattern of morphological
  evolution.
\newblock {\em Paleobiology}, 12(4):343--354.

\bibitem[LeCun et~al., 2015]{lecun2015}
LeCun, Y., Bengio, Y., and Hinton, G. (2015).
\newblock Deep learning.
\newblock {\em nature}, 521(7553):436--444.

\bibitem[LeCun et~al., 1998]{LecunBengioBottou1998SGDbetter}
LeCun, Y., Bottou, L., Bengio, Y., and Haffner, P. (1998).
\newblock Gradient-based learning applied to document recognition.
\newblock {\em Proc. {IEEE}}, 86(11):2278--2324.

\bibitem[Li et~al., 2018]{LossLandscapesVis2018}
Li, H., Xu, Z., Taylor, G., Studer, C., and Goldstein, T. (2018).
\newblock Visualizing the loss landscape of neural nets.
\newblock In Bengio, S., Wallach, H.~M., Larochelle, H., Grauman, K.,
  Cesa{-}Bianchi, N., and Garnett, R., editors, {\em Advances in Neural
  Information Processing Systems 31: NeurIPS 2018}, pages 6391--6401.

\bibitem[Li, 2011]{hypercapref}
Li, S. (2011).
\newblock Concise formulas for the area and volume of a hyperspherical cap.
\newblock {\em Asian Journal of Mathematics and Statistics}, 4(1):66--70.

\bibitem[Lynch et~al., 1993]{lynch1993mutational}
Lynch, M., B{\"u}rger, R., Butcher, D., and Gabriel, W. (1993).
\newblock The mutational meltdown in asexual populations.
\newblock {\em Journal of Heredity}, 84(5):339--344.

\bibitem[Mnih et~al., 2014]{HardAttention2014Google}
Mnih, V., Heess, N., Graves, A., and Kavukcuoglu, K. (2014).
\newblock Recurrent models of visual attention.
\newblock In Ghahramani, Z., Welling, M., Cortes, C., Lawrence, N.~D., and
  Weinberger, K.~Q., editors, {\em Advances in Neural Information Processing
  Systems 27: NeurIPS 2014}, pages 2204--2212.

\bibitem[Mulayoff and Michaeli, 2020]{MulayoffM20}
Mulayoff, R. and Michaeli, T. (2020).
\newblock Unique properties of flat minima in deep networks.
\newblock In {\em Proceedings of the 37th International Conference on Machine
  Learning, {ICML} 2020}, volume 119 of {\em Proceedings of Machine Learning
  Research}, pages 7108--7118. {PMLR}.

\bibitem[Nguyen, 2019]{NguyenConnectivity2019}
Nguyen, Q. (2019).
\newblock On connected sublevel sets in deep learning.
\newblock In Chaudhuri, K. and Salakhutdinov, R., editors, {\em Proceedings of
  the 36th International Conference on Machine Learning, {ICML} 2019},
  volume~97 of {\em Proceedings of Machine Learning Research}, pages
  4790--4799. {PMLR}.

\bibitem[Nguyen et~al., 2021]{NguyenConnectivity2021}
Nguyen, Q., Br{\'{e}}chet, P., and Mondelli, M. (2021).
\newblock On connectivity of solutions in deep learning: The role of
  over-parameterization and feature quality.
\newblock {\em CoRR}, abs/2102.09671.

\bibitem[Ohta, 1992]{ohta1992nearly}
Ohta, T. (1992).
\newblock The nearly neutral theory of molecular evolution.
\newblock {\em Annual review of ecology and systematics}, 23(1):263--286.

\bibitem[Orr, 2002]{orr2002population}
Orr, H.~A. (2002).
\newblock The population genetics of adaptation: the adaptation of dna
  sequences.
\newblock {\em Evolution}, 56(7):1317--1330.

\bibitem[Orr, 2005]{orr2005review}
Orr, H.~A. (2005).
\newblock The genetic theory of adaptation: a brief history.
\newblock {\em Nature Reviews Genetics}, 6(2):119--127.

\bibitem[Orr, 2006]{orr2006distribution}
Orr, H.~A. (2006).
\newblock The distribution of fitness effects among beneficial mutations in
  fisher's geometric model of adaptation.
\newblock {\em Journal of theoretical biology}, 238(2):279--285.

\bibitem[Orvieto et~al., 2022]{Bach2022Anticorrelatednoise}
Orvieto, A., Kersting, H., Proske, F., Bach, F.~R., and Lucchi, A. (2022).
\newblock Anticorrelated noise injection for improved generalization.
\newblock {\em CoRR}, abs/2202.02831.

\bibitem[Pan and Yang, 2010]{2010TransferLeareningReview}
Pan, S.~J. and Yang, Q. (2010).
\newblock A survey on transfer learning.
\newblock {\em {IEEE} Trans. Knowl. Data Eng.}, 22(10):1345--1359.

\bibitem[Pickands, 1975]{Pickands1975}
Pickands, J.~I. (1975).
\newblock Statistical inference using extreme order statistics.
\newblock {\em The Annals of Statistics}, 3(1):119--131.

\bibitem[Salimans et~al., 2017]{Stutskver2017ESasRLalternative}
Salimans, T., Ho, J., Chen, X., Sidor, S., and Sutskever, I. (2017).
\newblock Evolution strategies as a scalable alternative to reinforcement
  learning.
\newblock {\em arXiv preprint arXiv:1703.03864}.

\bibitem[Silver et~al., 2017]{AlphaGo2017}
Silver, D., Schrittwieser, J., Simonyan, K., Antonoglou, I., Huang, A., Guez,
  A., Hubert, T., Baker, L., Lai, M., Bolton, A., Chen, Y., Lillicrap, T.~P.,
  Hui, F., Sifre, L., van~den Driessche, G., Graepel, T., and Hassabis, D.
  (2017).
\newblock Mastering the game of go without human knowledge.
\newblock {\em Nat.}, 550(7676):354--359.

\bibitem[Srivastava et~al., 2014]{DropOut2014}
Srivastava, N., Hinton, G.~E., Krizhevsky, A., Sutskever, I., and
  Salakhutdinov, R. (2014).
\newblock Dropout: a simple way to prevent neural networks from overfitting.
\newblock {\em J. Mach. Learn. Res.}, 15(1):1929--1958.

\bibitem[Such et~al., 2017]{2017UberGeneticAlgos}
Such, F.~P., Madhavan, V., Conti, E., Lehman, J., Stanley, K.~O., and Clune, J.
  (2017).
\newblock Deep neuroevolution: Genetic algorithms are a competitive alternative
  for training deep neural networks for reinforcement learning.
\newblock {\em CoRR}, abs/1712.06567.

\bibitem[Sutton, 1991]{Sutton1991ReinforcementLearning}
Sutton, R.~S. (1991).
\newblock Dyna, an integrated architecture for learning, planning, and
  reacting.
\newblock {\em {SIGART} Bull.}, 2(4):160--163.

\bibitem[Tenaillon, 2014]{tenaillon2014FGMUses}
Tenaillon, O. (2014).
\newblock The utility of fisher's geometric model in evolutionary genetics.
\newblock {\em Annual review of ecology, evolution, and systematics},
  45:179--201.

\bibitem[Tenaillon et~al., 2007]{tenaillon2007quantifying}
Tenaillon, O., Silander, O.~K., Uzan, J.-P., and Chao, L. (2007).
\newblock Quantifying organismal complexity using a population genetic
  approach.
\newblock {\em PloS one}, 2(2):e217.

\bibitem[Vaswani et~al., 2017]{Transformer2017}
Vaswani, A., Shazeer, N., Parmar, N., Uszkoreit, J., Jones, L., Gomez, A.~N.,
  Kaiser, L., and Polosukhin, I. (2017).
\newblock Attention is all you need.
\newblock In Guyon, I., Luxburg, U.~V., Bengio, S., Wallach, H., Fergus, R.,
  Vishwanathan, S., and Garnett, R., editors, {\em Advances in Neural
  Information Processing Systems 30}, page 5998–6008. Curran Associates, Inc.

\bibitem[Vinyals et~al., 2019]{AlphaStar2019}
Vinyals, O., Babuschkin, I., Czarnecki, W.~M., Mathieu, M., Dudzik, A., Chung,
  J., Choi, D.~H., Powell, R., Ewalds, T., Georgiev, P., Oh, J., Horgan, D.,
  Kroiss, M., Danihelka, I., Huang, A., Sifre, L., Cai, T., Agapiou, J.~P.,
  Jaderberg, M., Vezhnevets, A.~S., Leblond, R., Pohlen, T., Dalibard, V.,
  Budden, D., Sulsky, Y., Molloy, J., Paine, T.~L., G{\"{u}}l{\c{c}}ehre,
  {\c{C}}., Wang, Z., Pfaff, T., Wu, Y., Ring, R., Yogatama, D., W{\"{u}}nsch,
  D., McKinney, K., Smith, O., Schaul, T., Lillicrap, T.~P., Kavukcuoglu, K.,
  Hassabis, D., Apps, C., and Silver, D. (2019).
\newblock Grandmaster level in starcraft {II} using multi-agent reinforcement
  learning.
\newblock {\em Nat.}, 575(7782):350--354.

\bibitem[Watkins and Dayan, 1992]{Watkins1992QLearning}
Watkins, C. J. C.~H. and Dayan, P. (1992).
\newblock Technical note q-learning.
\newblock {\em Mach. Learn.}, 8:279--292.

\bibitem[Wierstra et~al., 2008]{2008NES}
Wierstra, D., Schaul, T., Peters, J., and Schmidhuber, J. (2008).
\newblock Natural evolution strategies.
\newblock In {\em Proceedings of the {IEEE} Congress on Evolutionary
  Computation, {CEC} 2008, June 1-6, 2008, Hong Kong, China}, pages 3381--3387.
  {IEEE}.

\bibitem[Williams, 1992]{Wiliam1992PolicyGradient}
Williams, R.~J. (1992).
\newblock Simple statistical gradient-following algorithms for connectionist
  reinforcement learning.
\newblock {\em Mach. Learn.}, 8:229--256.

\bibitem[Wright, 1932]{Wright1932}
Wright, S. (1932).
\newblock The roles of mutation, inbreeding, crossbreeding and selection in
  evolution.
\newblock {\em Proceedings of the XI International Congress of Genetics},
  8:209--222.

\bibitem[Wu et~al., 2020]{Wu2020GeneralNoiseSGD}
Wu, J., Hu, W., Xiong, H., Huan, J., Braverman, V., and Zhu, Z. (2020).
\newblock On the noisy gradient descent that generalizes as {SGD}.
\newblock In {\em Proceedings of the 37th International Conference on Machine
  Learning, {ICML} 2020}, volume 119 of {\em Proceedings of Machine Learning
  Research}, pages 10367--10376. {PMLR}.

\bibitem[Xie et~al., 2021]{xie2020diffusion}
Xie, Z., Sato, I., and Sugiyama, M. (2021).
\newblock A diffusion theory for deep learning dynamics: Stochastic gradient
  descent exponentially favors flat minima.
\newblock In {\em 9th International Conference on Learning Representations,
  {ICLR} 2021, Virtual Event, Austria, May 3-7, 2021}. OpenReview.net.

\bibitem[Xing et~al., 2018]{Bengio2018WalkWithSGD}
Xing, C., Arpit, D., Tsirigotis, C., and Bengio, Y. (2018).
\newblock A walk with {SGD}.
\newblock {\em CoRR}, abs/1802.08770.

\bibitem[Xu et~al., 2015]{HardAttention2015Bengio}
Xu, K., Ba, J., Kiros, R., Cho, K., Courville, A.~C., Salakhutdinov, R., Zemel,
  R.~S., and Bengio, Y. (2015).
\newblock Show, attend and tell: Neural image caption generation with visual
  attention.
\newblock In Bach, F.~R. and Blei, D.~M., editors, {\em Proceedings of the 32nd
  International Conference on Machine Learning, {ICML} 2015}, volume~37 of {\em
  {JMLR} Workshop and Conference Proceedings}, pages 2048--2057. JMLR.org.

\bibitem[Yosinski et~al., 2014]{Bengio2014TransferLearning}
Yosinski, J., Clune, J., Bengio, Y., and Lipson, H. (2014).
\newblock How transferable are features in deep neural networks?
\newblock In Ghahramani, Z., Welling, M., Cortes, C., Lawrence, N.~D., and
  Weinberger, K.~Q., editors, {\em Advances in Neural Information Processing
  Systems 27: NeurIPS 2014}, pages 3320--3328.

\bibitem[Zhang et~al., 2020]{IdentityCrisisBengio2020}
Zhang, C., Bengio, S., Hardt, M., Mozer, M.~C., and Singer, Y. (2020).
\newblock Identity crisis: Memorization and generalization under extreme
  overparameterization.
\newblock In {\em 8th International Conference on Learning Representations,
  {ICLR} 2020, Addis Ababa, Ethiopia, April 26-30, 2020}. OpenReview.net.

\bibitem[Zhang et~al., 2021]{SBengioMemorization2021}
Zhang, C., Bengio, S., Hardt, M., Recht, B., and Vinyals, O. (2021).
\newblock Understanding deep learning (still) requires rethinking
  generalization.
\newblock {\em Commun. {ACM}}, 64(3):107--115.

\bibitem[Zhou et~al., 2019]{Zhou2019}
Zhou, M., Liu, T., Li, Y., Lin, D., Zhou, E., and Zhao, T. (2019).
\newblock Toward understanding the importance of noise in training neural
  networks.
\newblock In Chaudhuri, K. and Salakhutdinov, R., editors, {\em Proceedings of
  the 36th International Conference on Machine Learning, {ICML} 2019},
  volume~97 of {\em Proceedings of Machine Learning Research}, pages
  7594--7602. {PMLR}.

\bibitem[Zhu et~al., 2018]{zhu2018anisotropicnoise}
Zhu, Z., Wu, J., Yu, B., Wu, L., and Ma, J. (2018).
\newblock The anisotropic noise in stochastic gradient descent: Its behavior of
  escaping from sharp minima and regularization effects.
\newblock {\em arXiv preprint arXiv:1803.00195}.

\bibitem[Ziyin et~al., 2021a]{Ueda2022LocalMaximaSGD}
Ziyin, L., Li, B., Simon, J.~B., and Ueda, M. (2021a).
\newblock Sgd can converge to local maxima.
\newblock In {\em International Conference on Learning Representations}.

\bibitem[Ziyin et~al., 2021b]{Ueda2021strengthOfMinibatchNoise}
Ziyin, L., Liu, K., Mori, T., and Ueda, M. (2021b).
\newblock Strength of minibatch noise in sgd.
\newblock {\em arXiv preprint arXiv:2102.05375}.

\end{thebibliography}

\newpage

\section{Appendix\\Evolutionary Algorithms in the Light of SGD:\\
Limit Equivalence, Minima Flatness, and Transfer Learning}

This is the appendix to the "Evolutionary Algorithms in the Light of SGD:
Limit Equivalence, Minima Flatness, and Transfer Learning" paper currently in review.

\section{Probability of Randomly Sampling a Vector Close to a Target Vector}

Assuming the sampling is performed uniformly on a unit sphere, the probability to sample a vector within $\alpha$ degrees from a desired vector follows a regularized Beta-Incomplete function $\frac{1}{2} I_{sin^2\alpha}(\frac{n-1}{2}, \frac{1}{2})$. To visualize this function, we plotted the probability of a random sampling vector to land within $\alpha$ degrees from a reference vector as a function of angle $\alpha$ and dimension $n$ on the Appendix Fig.\ref{fig:beta_inc_sweep}.

\begin{figure}[htbp]
\vskip 0.2in
\begin{center}
\centerline{\includegraphics[width=\columnwidth]{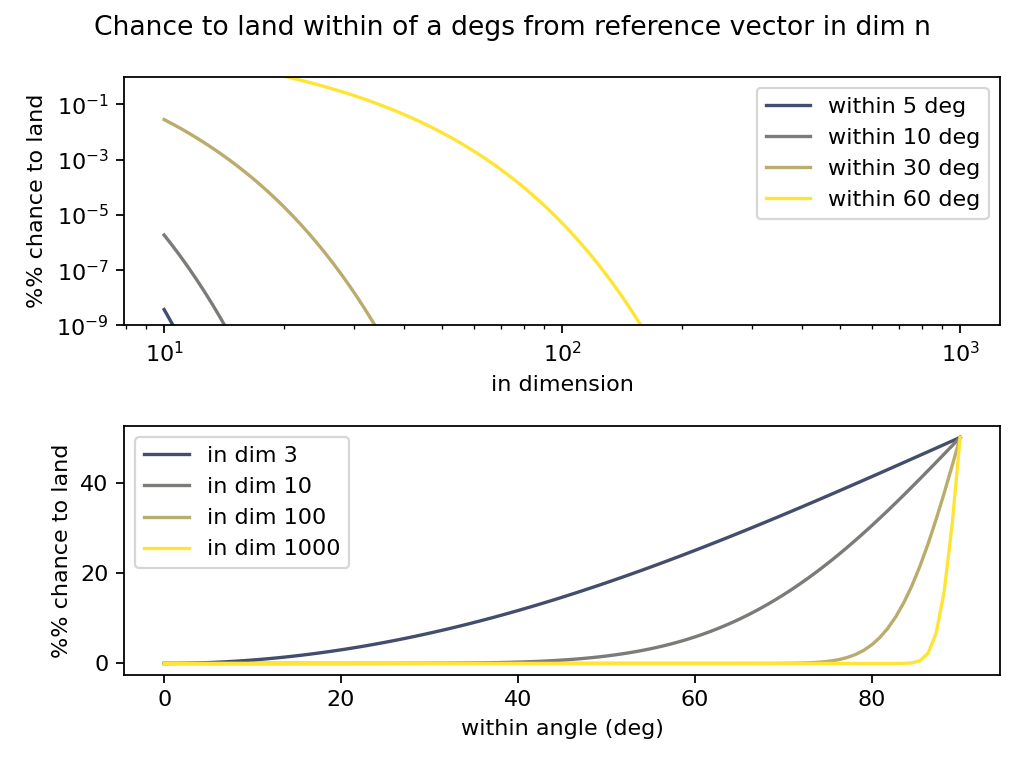}}
\caption{The probability (\%) to land within $\alpha$ degrees from the reference vectors as a function of dimensions number for fixed angles and as a function of angles for fixed dimensions}
\label{fig:beta_inc_sweep}
\end{center}
\vskip -0.2in
\end{figure}

\section{A Trivial Algorithm in the GO-EA Class}

A trivial example of an algorithm in GO-EA class is fixed population greedy random search, formally referred to as ($1,\lambda$)-Evolutionary Search in literature and also known as Evolutionary Strategies. The pseudo-code for the exact algorithm we used is presented in Appendix Alg.\ref{basic_go_ea}. Namely, the version of GO-EA we use does not contain any drift or noise.

\begin{algorithm}[htbp]
  \caption{Fixed population greedy random search}
  \label{basic_go_ea}
  \textbf{Input}: loss function $\mathcal{L}_{\bm \theta}$, starting parameter $\theta_0$\\
  \textbf{Parameter}: update distance $\epsilon$, population size $N$, generations $M$\\
  \textbf{Output}:  optimal parameters found $\theta_{fin}$\\
\begin{algorithmic}[1]
  \STATE Initialize $ \theta = \theta_0$
  \FOR{m in [1, .., M]}
      \FOR{n in [1, ..., N]}
          \STATE choose random update direction $\theta_{rand}$\\
          \IF{$\mathcal{L}_{\theta + \epsilon \theta_{rand}} <\mathcal{L}_{\theta}$}
            \STATE $\theta = \theta + \epsilon \theta_{rand}$.
          \ENDIF
      \ENDFOR
  \ENDFOR
  \STATE return $\theta_{fin}=\theta$
\end{algorithmic}
\end{algorithm}

\section{Neural Network Architecture}

Our ConvNet has 9854 trainable parameters with specific architecture of 784-\textbf{[8C5-P2]}-\textbf{[16C5-P2]}-\textbf{[8C3]}-\textbf{[4C3]}-\textbf{[196L24]}-\textbf{[24L10]}-10. \textbf{8C5} denotes a convolution layer with an 8 feature map using a 5x5 filter with stride 1 and padding 2, \textbf{P2} is max polling with 2x2 filter and stride 2, and \textbf{24L10} is a linear layer mapping from 24 to 10 channels. Every bloc (\textbf{[]}) is followed by a ReLU activation function, and a drop-out normalization is applied to it (10\% by default, abbreviated to DO:0.1), except for the last output layer. The input image is partially occluded by using a 10\% drop-out on input (DI:0.1). The width of the layers is controlled by two key variables: latent maps in the first layer (LM:8) with subsequent ones following a predefined ratio, and linear width (LW:24), determining the number of neurons in the hidden linear layer.

In SGD mode, the model is trained using cross-entropy loss and SGD optimizer with a learning rate of 0.002 and no momentum, with a batch size of 32 (B:32) and 10 epochs (E:10) unless specified otherwise. The abbreviations above are used in figures, with \textbf{x} or \textbf{/} denoting the multiplication or division of hyperparameters by a factor (e.g. LWx2) compared to the reference and \textbf{:} defining parameter being set to a value (e.g. DO/DI: .0/.0). Finally, \textbf{-} means the parameter was set to 0 (e.g. -DO). All figures can be re-generated from code and run records provided, including in greyscale-uniform colormap.

\section{An Simple Algorithm in the GO-EA Is Able to Train a ConvNet On MNIST}


\begin{figure}[h]
\vskip 0.2in
\begin{center}
\centerline{\includegraphics[width=\columnwidth]{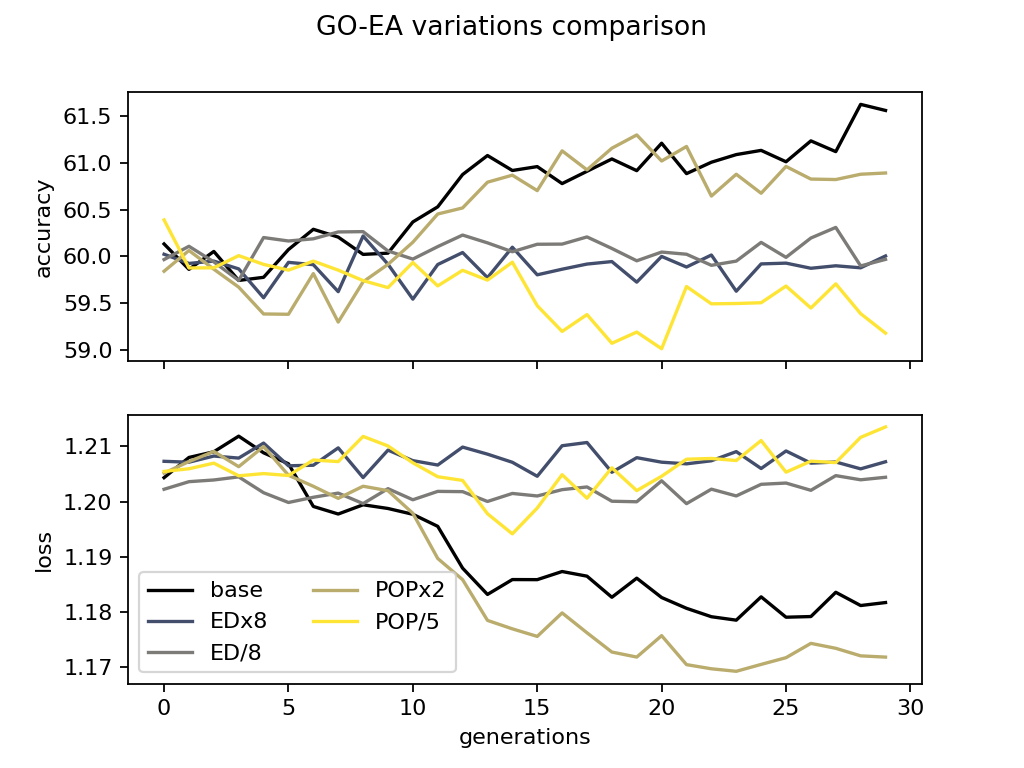}}
\caption{Variation of the speed of convergence of our base ConvNet on MNIST dataset based on the hyperparameters of our implementation of GO-EA class algorithm}
\label{go_ea_sweep}
\end{center}
\vskip -0.2in
\end{figure}

Given that optimized, tested, and scalable Evolutionary Algorithms in the GO-EA class have already been described and made available, notably by \citep{2017UberGeneticAlgos}, here we limited ourselves only to a proof-of-concept implementation. Specifically, due to the experimental infrastructure limitation, forcing a sequential sampling, we are sampling a population of 20 individuals per generation, evaluating the fitness (cross-entropy loss) on the whole training part of the MNIST dataset. Given the lack of optimization of code used to implement an instance of GO-EA class, each population sampling round takes slightly over a minute, leading to a slow - comparatively to SGD - performance. 

However, as Fig.\ref{go_ea_training} demonstrates, in about 800 updates, our GO-EA instance (edit distance ED:0.01, population POP:20) is able to take our 9854 parameters ConvNet from 10\% accuracy on MNIST to 60\% accuracy. In a single epoch with base hyperparameters, SGD performed over 1800 updates with an accuracy of only about 15\% on the training set and will need 1800 more to get to 85\% in the following epoch, albeit at a lower learning rate. Unlike SGD, our instance of GO-EA did not perform back-propagation, and the population is fully and efficiently scalable through random seed sharing, as explained by \citep{2017UberGeneticAlgos}.

While increasing the population size is a possible way to increase the speed of convergence, the behavior of other hyperparameters seems to be negligible, as a rapid exploration of parameter space over 30 generations near 60\% accuracy suggests in Appendix Fig.\ref{go_ea_sweep}.

\subsection{Model Validity}

\begin{figure}[t]
\centering
\includegraphics[width=0.9\columnwidth]{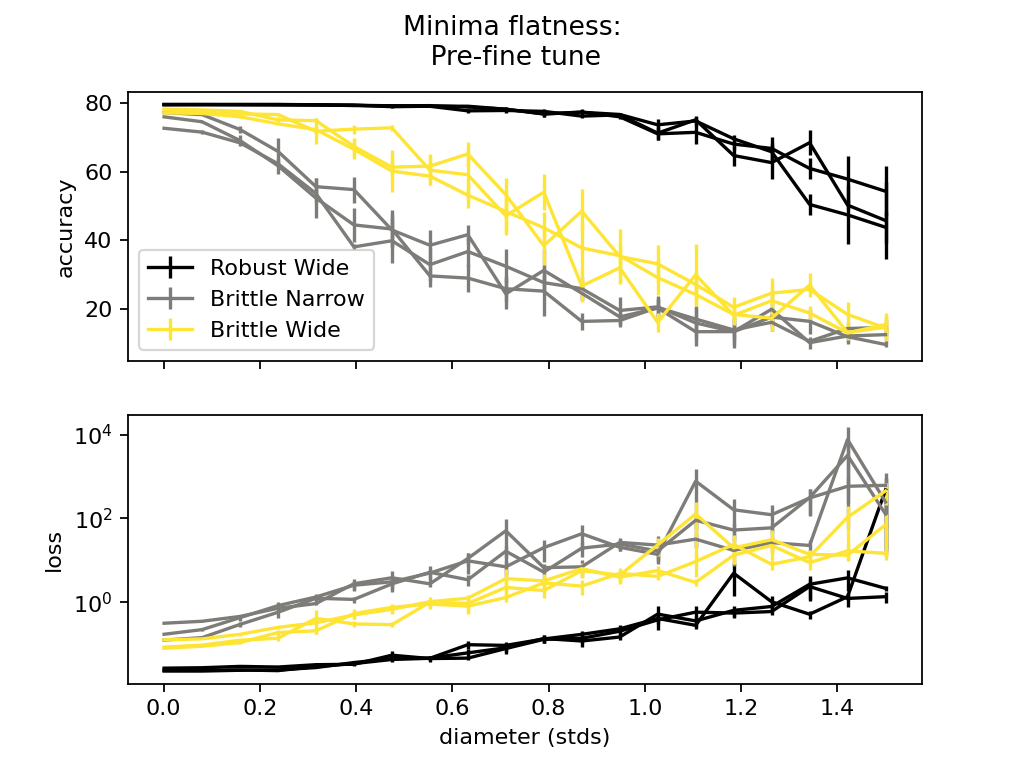} 
\caption{8 class minima flatness for archetype models. Due to how accuracy is calculated, the best-expected model performance is 80\%. Error bars derived from sampling replicas, lines are full independent replicas.}.
\label{flat_8_cl}
\end{figure}

The base model we used is able to achieve an error rate below 2\% on the validation set after 20 epochs of training and about 2.5\% in 10 epochs. Our approach to measuring the ruggedness of loss landscape and minima flatness was consistent with prior measures and recapitulated the results by \citet{LossLandscapesVis2018}, which we focus on due to its applicability in the non-differentiable setting. Specifically, flatter minima resulted from more overparametrized models, drop-out normalization, and smaller batches (cf Appendix Figs. \ref{flat_overparametrization}, \ref{flat_dropout}, \ref{flat_batch}). We also confirmed that the total batch number in training was not affecting the minima flatness (cf. Appendix Fig.\ref{flat_epochs}).

\newpage
\section{Minima Flatness Profiles}

In order to obtain minima flatness profiles, several models with identical hyperparameters are trained independently. As described by \citet{LossLandscapesVis2018}, we measure the minima flatness or the local smoothness by selecting a random vector from an N-dimensional Gaussian distribution with covariance defined by the Identity matrix scaled for the standard deviation of activations in a given layer. Once a vector is selected at random, the model stability is recorded along the axis of its extension, performing a "sweep" of the fitness landscape in the neighborhood of the model parameter vector $\theta$.

In the diagrams of the minima flatness we present, the $x$ axis is in standard deviations of activations. For completeness, we present the degradation of both loss and accuracy, and for precision, for each parameter vector, we sample at least 5 independent filter-wise normalized random vectors, resulting in error bars when the error at a given distance for them diverges (sampling replicas). Each line corresponds to a completely independent parameter vector from a model trained from scratch with the same hyperparameters (full replicas).

\newpage

\subsection{Dropout, Overparametrization, Small Batches, but Not Length of Training Contribute to Flatter Minima}

\begin{figure}[H]
\vskip 0.2in
\begin{center}
\centerline{\includegraphics[width=\columnwidth]{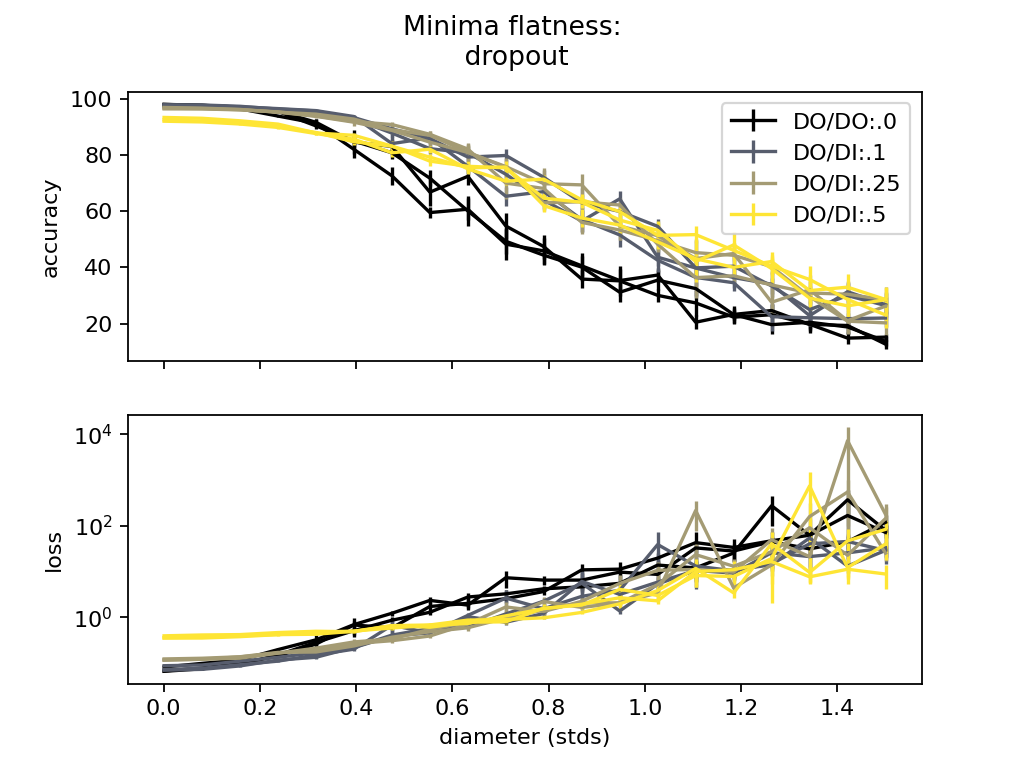}}
\caption{Effect of dropout and input dropout on minima flatness}
\label{flat_dropout}
\end{center}
\vskip -0.2in
\end{figure}

\begin{figure}[H]
\vskip 0.2in
\begin{center}
\centerline{\includegraphics[width=\columnwidth]{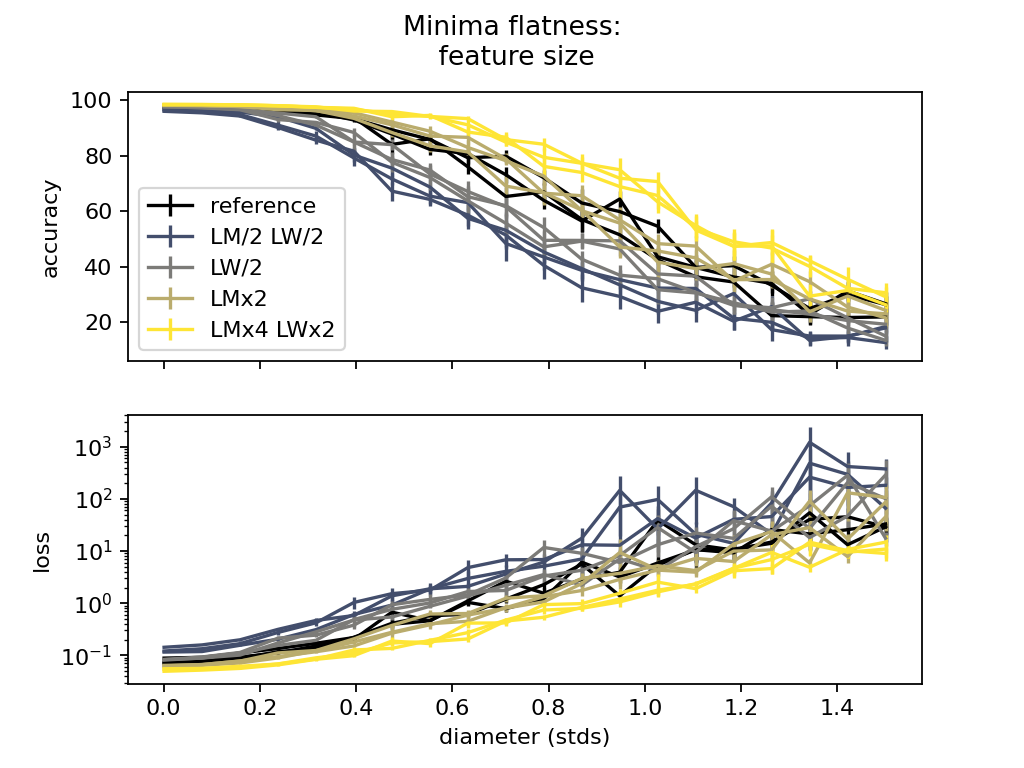}}
\caption{Effect of under- and over- parametrization of the ConvNet on the minima flatness}
\label{flat_overparametrization}
\end{center}
\vskip -0.2in
\end{figure}

\begin{figure}[H]
\vskip 0.2in
\begin{center}
\centerline{\includegraphics[width=\columnwidth]{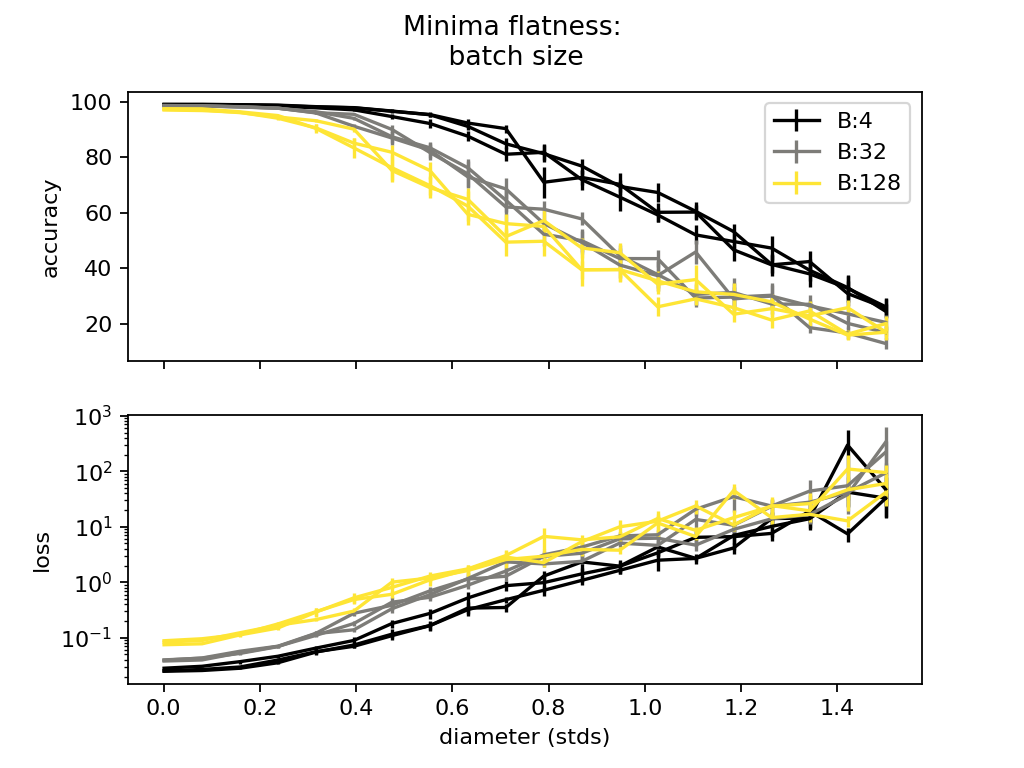}}
\caption{Effect of the small and large batch sizes on the minima flatness}
\label{flat_batch}
\end{center}
\vskip -0.2in
\end{figure}

\begin{figure}[H]
\vskip 0.2in
\begin{center}
\centerline{\includegraphics[width=\columnwidth]{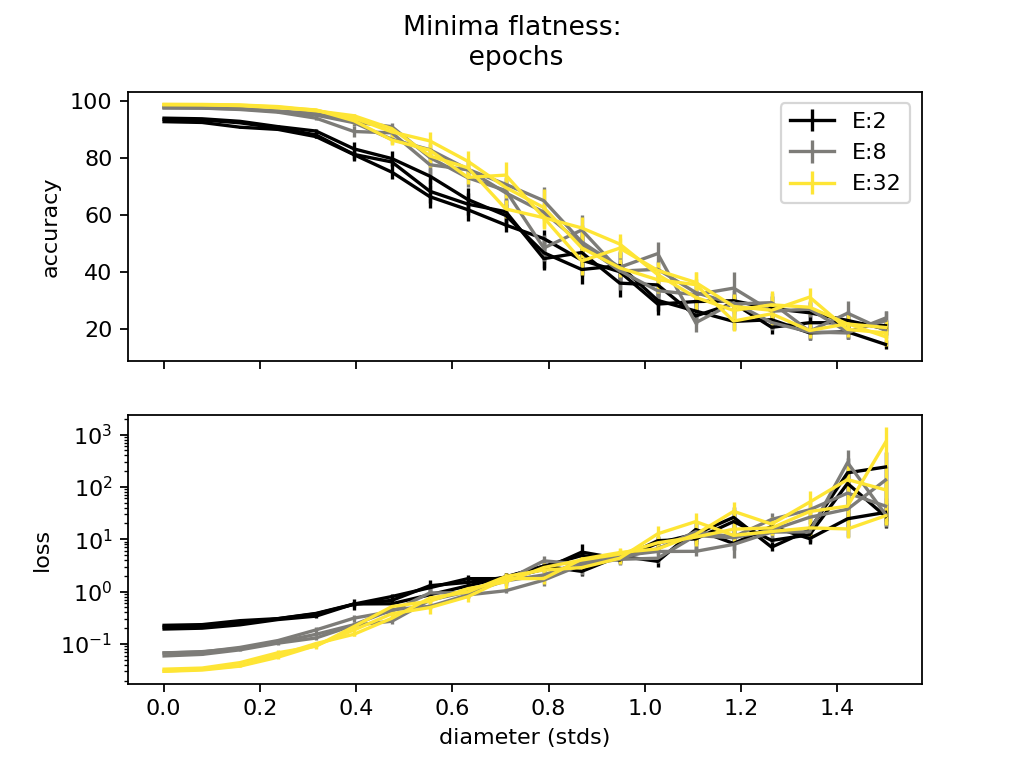}}
\caption{Larger number of updates with smaller batch size cannot explain flatter minima, given the length of training does not affect the minima flatness}
\label{flat_epochs}
\end{center}
\vskip -0.2in
\end{figure}

\newpage
\newpage
\subsection{Minima Flatness, Error Correction, Generalization, and Transfer Learning}

\begin{figure}[H]
\vskip 0.2in
\begin{center}
\centerline{\includegraphics[width=\columnwidth]{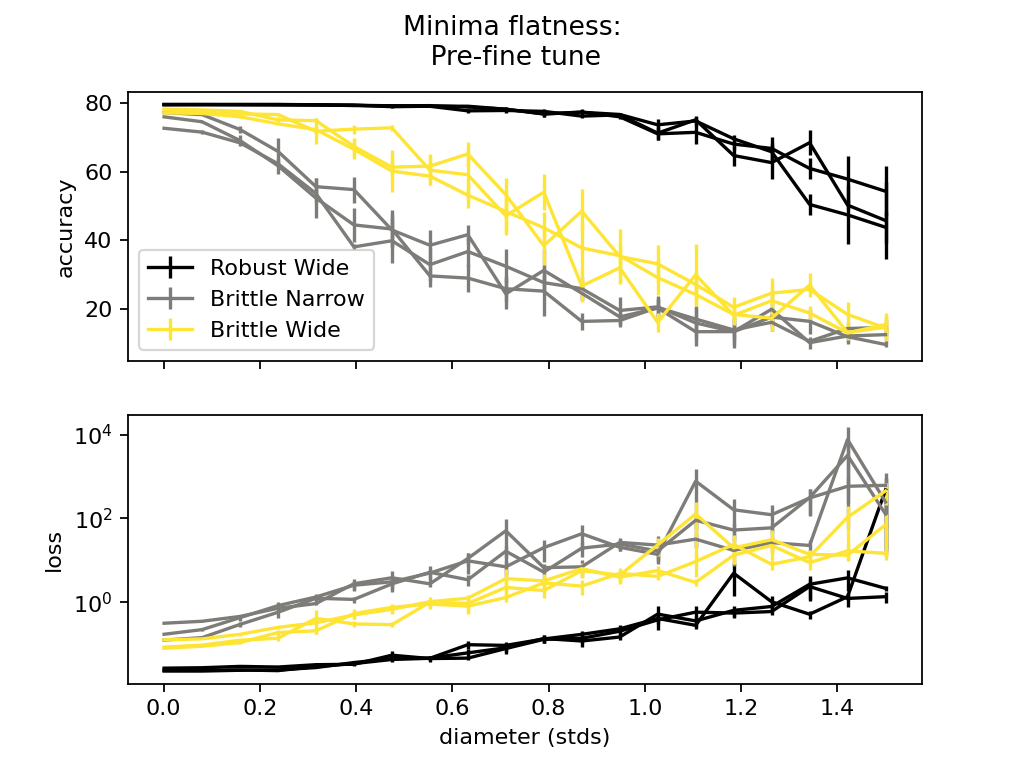}}
\caption{Minima flatness for 3 archetypes before transfer learning}
\label{flat_pre_fine_tune}
\end{center}
\vskip -0.2in
\end{figure}

\begin{figure}[H]
\vskip 0.2in
\begin{center}
\centerline{\includegraphics[width=\columnwidth]{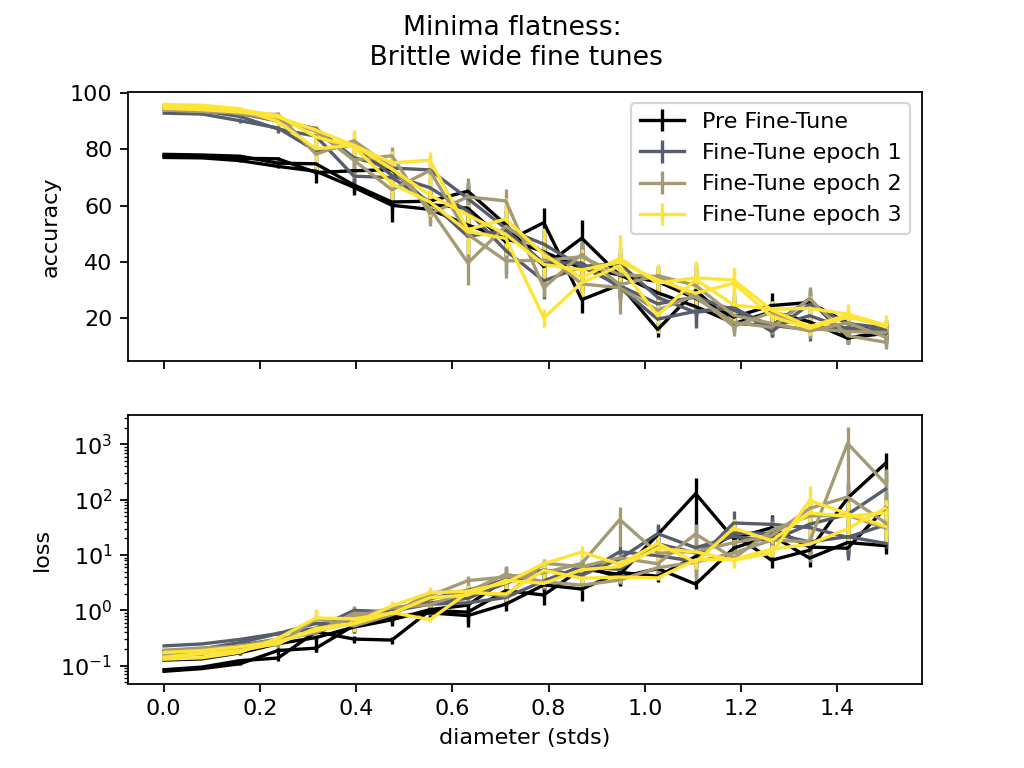}}
\caption{Change in minima flatness for the Brittle Wide archetype during transfer learning}
\label{flat_fine_tune_epochs}
\end{center}
\vskip -0.2in
\end{figure}

\begin{figure}[H]
\vskip 0.2in
\begin{center}
\centerline{\includegraphics[width=\columnwidth]{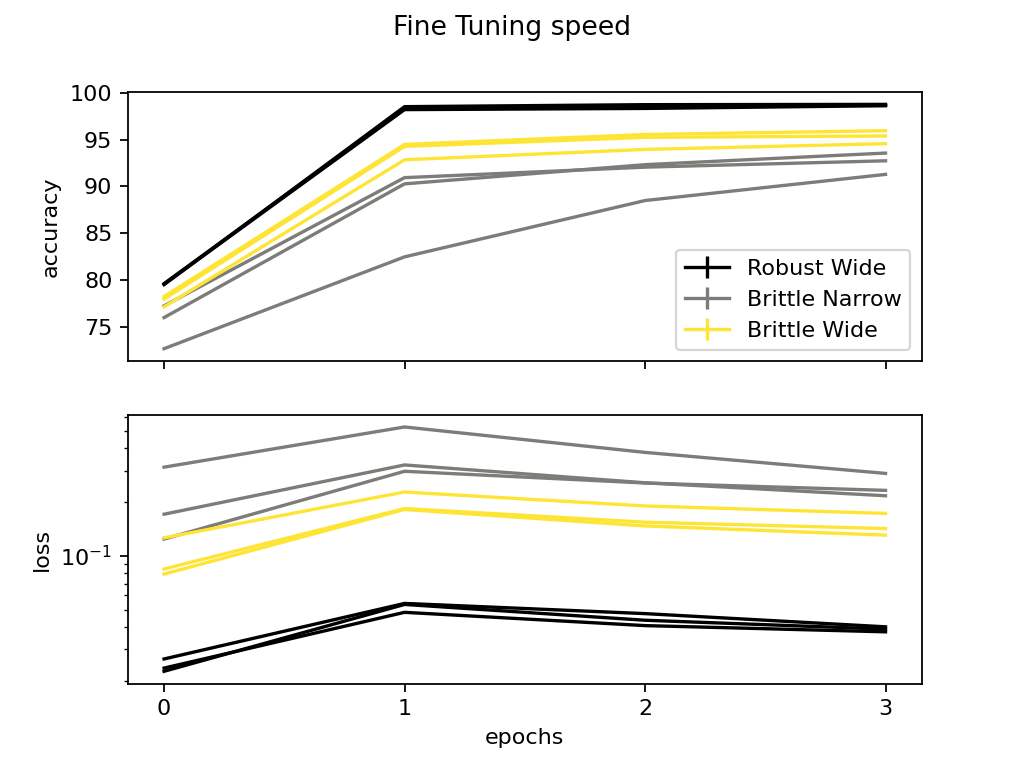}}
\caption{Change in accuracy and loss over epochs of transfer learning for the three archetype models. Minima flatness does not seem to affect transfer learning speed.  Each line is a full independent replicate.}
\label{flat_fine_tune_speed}
\end{center}
\vskip -0.2in
\end{figure}




\newpage

\section{Experimental Setup and Computation Time}

All the methods are evaluated on a workstation equipped with Intel Core i9-9900K (8 cores/16 threads CPU), 64 Gb of RAM clocked at 2666 MHz, 2 TB NVME M.2 SSD, and an RTX 3080 graphics cards, running an Ubuntu 20.04 LTS distribution. The evaluations were performed within a Docker container, Docker Community Edition, version 20.10.12. The code used Miniconda version 4.12.0 and Python 3.9 with CUDA version 11.3.1. More detailed information is available in the \verb|requirements.txt| found in the code repository shared with this Appendix, \href{https://github.com/chiffa/ALIFE2023_GOEA-SGD}{https://github.com/chiffa/ALIFE2023\_GOEA-SGD}.

The total compute time for re-running experiments, except for the GO-EA training trace, is under 2 hours. GO-EA, due to non-optimized deterministic random update vector derivation from seed, shuttles data between CPU and GPU for every individual in every generation. Due to the overhead introduced, the GO-EA trace presented in Fig.\ref{go_ea_training} took approximately 24 hours.

No hyperparameter optimization outside model versions presented here were presented. A common learning rate of 0.02 for training ConvNets on MNIST was divided by 10 to accommodate very small batches, resulting in 0.002. The same rate was also used for GO-EA training.

\end{document}